\documentclass{article}

\PassOptionsToPackage{numbers, sort, compress}{natbib}

\usepackage[main,final]{neurips_2026}

\usepackage[utf8]{inputenc} 
\usepackage[T1]{fontenc}    
\usepackage{url}            
\usepackage{booktabs}       
\usepackage{amsfonts}       
\usepackage{nicefrac}       
\usepackage{xcolor}         
\usepackage{caption}
\usepackage{footmisc}
\usepackage{array}
\usepackage{enumitem}
\usepackage{algorithmic}
\usepackage{algorithm}
\usepackage{graphicx}
\usepackage{textcomp}
\usepackage{subcaption}
\usepackage{minitoc}
\usepackage{tablefootnote}
\newcommand{\myheight}{3.6cm}
\newcommand{\fourheight}{3.8cm}
\newcommand{\contaheight}{3.6cm}
\newcommand{\mywidth}{0.95\linewidth}

\usepackage{times}
\usepackage{graphicx}
\usepackage{latexsym}
\usepackage{amsmath,amsthm,amssymb,stmaryrd,pifont,wasysym}
\usepackage{dsfont}
\usepackage{hyperref}
\usepackage{tikz}
\usetikzlibrary{positioning}

\newtheorem{definition}{Definition}
\newtheorem{lemma}{Lemma}
\newtheorem{assertion}{Assertion}
\newtheorem{proposition}{Proposition}
\newtheorem{theorem}{Theorem}
\newtheorem{corollary}{Corollary}
\newtheorem{example}{Example}
\newtheorem{remark}{Remark}

\definecolor{amber}{rgb}{1.0, 0.75, 0.0}
\definecolor{airforceblue}{rgb}{0.36, 0.54, 0.66}

\DeclareMathOperator*{\argmin}{argmin}

\newcommand{\n}{\llbracket n \rrbracket}
\newcommand{\N}{\llbracket N \rrbracket}
\newcommand{\Sn}{\mathfrak{S}_n}

\newcommand{\trinorm}[1]{\left\vert\kern-0.25ex\vert\kern-0.25ex\vert#1 \vert\kern-0.25ex\vert\kern-0.25ex\vert\right.}
\def\comments#1
{
\textcolor{red}{#1}
}

\title{Decentralized Ranking Aggregation via Gossip: Convergence and Robustness}
\author{
  Kerrian Le Caillec\thanks{Main contributors.}\quad Anna van Elst\footnotemark[1]\quad Igor Colin\quad Stephan Clémençon \\
  LTCI, T\'el\'ecom Paris, Institut Polytechnique de Paris\\
  \texttt{name.surname@telecom-paris.fr}
}

\begin{document}

\maketitle

\begin{abstract}
The concept of ranking aggregation plays a central role in preference analysis, and numerous algorithms for calculating median rankings, often originating in social choice theory, have been documented in the literature, offering theoretical guarantees in a centralized setting, \textit{i.e.}, when all the ranking data to be aggregated can be brought together in a single computing unit. For many technologies (\textit{e.g.}, peer-to-peer networks, IoT, multi-agent systems), extending the ability to calculate consensus rankings with guarantees of convergence and resilience to potential contamination in a decentralized setting, when preference data is initially distributed across a communicating network, remains a major methodological challenge. Indeed, in recent years, the literature on decentralized computation has mainly focused on computing or optimizing statistics such as arithmetic means using gossip algorithms. The purpose of this article is precisely to study how to achieve reliable and resilient consensus on collective rankings in a decentralized setting, thereby raising new questions, such as robustness to corrupted nodes and scalability through reduced communication costs. The approach proposed and analyzed here relies on randomized gossip algorithms, which allow autonomous agents to compute a global ranking consensus using local communications only, without coordination or a central authority. 

\end{abstract}

\section{Introduction}
 
\textit{Ranking aggregation} aims to summarize several rankings relating to a common set of alternatives, obtained from different sources, into a single representative or \textit{median ranking}. This problem, also known as \textit{consensus ranking}, has its origin in social choice theory \cite{borda1781m, condorcet1785essai, Copeland1951, Kemeny59} and voting systems \cite{brandt2012computational}. More recently, rank aggregation has emerged as a core primitive in machine learning. In learning to rank \cite{liu_learning_2009}, the goal is to train models from ranked supervision signals arising from multiple annotators or queries, and aggregating those signals consistently is a rank aggregation problem. In preference learning \cite{furnkranz_preference_2011}, models learn utility functions from pairwise comparisons, and consolidating preferences across distributed data sources requires principled aggregation. Most prominently, Reinforcement Learning from Human Feedback  \cite{christiano_deep_2017, ziegler_fine-tuning_2020}, especially in Natural Language Processing tasks, relies on aggregating ranked preferences from multiple human annotators to shape reward models, yet standard pipelines collapse this aggregation into a single centralized step, ignoring the social choice structure of the problem and its vulnerabilities \cite{siththaranjan_distributional_2024, colombo_what_2022}.
Other applications include crowdsourcing platforms \cite{chatterjee2018weighted, caragiannis2019optimizing}, and recommender systems \cite{baltrunas2010group}. While the theoretical and empirical properties of centralized rank aggregation methods are now well established \cite{korba2017learning, dwork2001rank}, \textit{decentralized} ranking aggregation has received much less attention in the literature. 

To our knowledge, no prior work establishes convergence guarantees for fully decentralized rank aggregation; in particular, federated settings \cite{mcmahan_communication-efficient_2017, sima2024federated} require a central aggregator. Indeed, existing work has mainly focused on configurations which are inherently centralized: they require a single trusted entity to collect and process all individual preferences, introducing a single point of failure and communication bottlenecks. In this paper, we consider a fully decentralized framework in which voters are distributed across a communication network, have limited computing resources, and can only communicate with their immediate neighbors, with no access to a central server \cite{kempe_gossip-based_2003, boyd2006randomized}. This model is standard in \textit{e.g.}, peer-to-peer systems \cite{badis2021p2pcf}, and in wireless sensor networks and IoT networks \cite{shah2009gossip}. Our goal is to study how a network of autonomous agents, each with its own preference ranking, can reach a consensus ranking without a central authority. Unlike real-valued quantities, rankings are elements of the symmetric group $\Sn$, and the coordinate-wise arithmetic mean does not preserve their permutation structure. In this paper, we study a general class of score-based ranking consensus methods, subsuming many classical voting rules such as Borda~\cite{borda1781m} and Copeland~\cite{Copeland1951} as canonical instances~\cite{Brandt_Conitzer_Endriss_Lang_Procaccia_2016}. Our approach applies a single decentralized gossip algorithm~\cite{boyd2006randomized} for score estimation across any rule in this class, followed by a local ranking projection onto the symmetric group $\Sn$, yielding a unified and communication-efficient framework for decentralized consensus. A key consideration is robustness against corrupt nodes and vote manipulation, a topic explored through the notions of spam resistance~\cite{dwork2001rank} and breakdown points~\cite{goibert2023robust}.

\paragraph{Contributions.} Our main contributions are as follows. \textit{(i)} We propose a unified gossip algorithm for decentralized score-based consensus applicable to any voting rule following a specific decomposition. We show that the average Kendall-$\tau$ error decays exponentially fast under any locally ranking-stable aggregation map, at a rate depending on a stability radius. \textit{(ii)} We provide an empirical robustness analysis via the breakdown function and then we relate contamination levels to the gossip process. \textit{(iii)} We present experiments on various network topologies and ranking datasets, showing how Borda and Copeland consensus exhibit robustness under data contamination (see Figure~\ref{fig:gossip_desc} for illustration of the protocols).

\paragraph{Other related work.}
Data ranking over gossip aggregation has recently attracted attention from a statistical perspective~\cite{van2026robust, van2026fast}, establishing connections between gossip-based averaging and robust estimation for rankings and trimmed means. Byzantine robustness in decentralized learning has been studied extensively~\cite{he_byzantine-robust_2023, gaucher_unified_2025}, with particular attention to characterizing the influence of corrupted nodes on convergence; the dual-approach framework in~\cite{gaucher_byzantine-robust_2025} provides tight clipping-based guarantees for average consensus that complement our breakdown function analysis in Section~\ref{sec:robustness}. \citet{jiao2016controlling} and \citet{jiao2015kendall} approximate the Kemeny consensus using an embedding for permutations based on the usual voting rules.

\section{Background and preliminaries}\label{sec:background}

We briefly review the fundamental concepts of consensus ranking theory and gossip-based decentralized computing. Throughout this article, preferences over $n$ items indexed by $i \in \n:=\{1, \ldots, n\}$ are represented as permutations ${\sigma} \in \Sn$, where $\Sn$ is the symmetric group and $\sigma(i)$ denotes the rank assigned to item $i$. We write $\mathbb{I}\{\mathcal{E}\}$ for the indicator function of any event $\mathcal{E}$. We denote scalars by lowercase letters \(x \in \mathbb{R}\), vectors by boldface lowercase letters \(\mathbf{x} \in \mathbb{R}^n\), and matrices by boldface uppercase letters \(\mathbf{X} \in \mathbb{R}^{n \times m}\). For $d\in\mathbb{N}$, we write $\{\mathbf{e}_k\mid k\in \llbracket d \rrbracket\}$ for the canonical basis of $\mathbb{R}^d$, \(\mathbf{M}^{\top}\) for the transpose of any matrix \(\mathbf{M}\), \(|F|\) for the cardinality of any finite set $F$, \(\mathbf{I}_d\) for the identity matrix in \(\mathbb{R}^{d \times d}\), \(\mathbf{1}_d\) for the vector in \(\mathbb{R}^d\) whose coordinates are all equal to $1$, and \(\|\cdot\|\) for the usual Euclidean ($\ell^2$-)norm. For any probability distribution $P$ on $\Sn$ and any map $f:\Sn \to \mathcal X$, where $\mathcal X$ is a Hilbert space, we write $\mathbb E_P f$ for $\mathbb E_{\Sigma\sim P} f(\Sigma)$ when there is no ambiguity.

\paragraph{Ranking aggregation/consensus.}
\label{sec:aggregation}
Let $P$ be a probability distribution over $\Sn$, written as $P = \sum_{\sigma \in \Sn} P(\sigma)\,\delta_\sigma,$ where $\delta_\sigma$ denotes the Dirac measure at $\sigma$. We denote by $\Delta_{\Sn}$ the probability simplex over $\Sn$. Given a measurable map $g : \Sn \times \Delta_{\Sn} \to \mathbb{R}$, the ranking aggregation problem \cite{korba2017learning} consists in finding a consensus permutation ${\sigma}_\star \in \Sn$ such that
\[{\sigma}_\star \in \arg\min_{\sigma \in \Sn} g(\sigma, P).\]
Let $(\sigma_v)_{v=1}^N$ be a dataset of $N$ permutations in $\mathfrak{S}_n$ such that $N\ge n$ (in practical cases, $n\ll N$), representing one ranking per voter, with repetitions allowed. The underlying distribution $P$ is then approximated by the empirical measure $\widehat{P}_N = N^{-1} \sum_{v=1}^N \delta_{\sigma_v}.$ Since the order of the voters is irrelevant, a dataset is a multiset of $N$ permutations, and we denote by $\mathrm{Sym}^N(\Sn)$ the set of all such multisets, and by $\mathcal P_N(\Sn)$ the set of the corresponding empirical distributions. This notion of consensus ranking can be interpreted in different ways. The consensus can correspond to the mode of a probability distribution on $\Sn$ (typically the Mallows distribution \cite{Mallows57}, see Appendix~\ref{sec:prob_model} for details), considering the rankings ${ \sigma}_v$ as i.i.d. realizations of it. Alternatively, the consensus can be defined as a barycenter of the dataset with respect to a metric on $\Sn$, an approach widely used in practice. In particular, for some (pseudo-)metric $d: \Sn \times \Sn \to \mathbb R_+$, a ranking consensus can be defined as any solution of the (discrete) optimization problem
\begin{equation}\label{eq:metric_median}
\min _{{ \sigma} \in \Sn} \mathbb E_{{\Sigma} \sim P}\, d({\sigma}, {\Sigma}).
\end{equation}
Among the many possible choices of $d$ (see, \textit{e.g.}, \cite{Deza}), a widely used one is the Kendall-$\tau$ distance $d_\tau$, defined as the number of discordant pairs: for any ${ \sigma}, { \sigma}' \in \Sn$,
\[d_\tau\left({ \sigma}, { \sigma}^{\prime}\right)=\sum_{1\le i<j\le n} \mathbb{I}\left\{({ \sigma}(i)-{ \sigma}(j))\left({ \sigma}^{\prime}(i)-{ \sigma}^{\prime}(j)\right)<0\right\}.\]
When $d=d_\tau$, any solution of Equation~\eqref{eq:metric_median} is referred to as a \textit{Kemeny consensus} \cite{Kemeny59} and, remarkably, the two approaches, probabilistic and metric, coincide: the Kemeny consensus is the maximum likelihood estimator of the central ranking of a Mallows model, under the assumption that the ${ \sigma}_v$ are i.i.d. realizations of it (see \cite{young1988condorcet}). From a social choice perspective, the Kemeny consensus satisfies the Condorcet criterion: any candidate $i\in \n$ who would win against all others in head-to-head majority votes is guaranteed to be ranked first (see \cite{Brandt_Conitzer_Endriss_Lang_Procaccia_2016}). The Kemeny consensus further satisfies the extended Condorcet property, which ensures robustness to adversarial inputs (see \cite{dwork2001rank}).  

However, these properties are achieved at a significant computational cost, as computing a Kemeny consensus is NP-hard in general (see, \textit{e.g.}, \cite{Hudry08}). This motivates score-based ranking aggregation methods. The principle is to define a score function $S : \mathcal P_N(\Sn) \to \mathbb{R}^n$ that maps a dataset to a vector of candidate scores $\mathbf{s} = (s_1, \dots, s_n)^\top$. The score vectors are then ranked in ascending/descending order depending on the score counting method, and this yields a unique consensus in $\Sn$ when there are no ties, \textit{i.e.,} when all score coordinates are distinct (see Appendix~\ref{sec:practical-tie} for details). For a given distribution $P$, we also denote the \emph{pairwise probabilities} for a pair of alternatives $(i,j)$ as $p_{ij}= \mathbb P(\sigma(i) < \sigma(j)) $.   We now present examples of score-based voting rules.
\begin{example}[Borda method, \cite{borda1781m}]
Each permutation ${ \sigma}_v\in\Sn$ is embedded as its rank vector $({\sigma}_v(1), \ldots, {\sigma}_v(n))^\top \in \mathbb{R}^n$, the Borda score function $S_B$ returns the Borda score by averaging the $N$ vectors and taking the opposite:
\[S_B\left(\widehat P_N\right) = -\frac{1}{N}\sum_{v=1}^N ({\sigma}_v(1), \ldots, {\sigma}_v(n))^\top.\]
The Borda consensus ${ \sigma}^B_\star \in \Sn$ ranks candidates in decreasing order of $[S_B(\widehat P_N)]_i$, \textit{i.e.,} in ascending order of their average ranks. This method exploits the positional information in each ranking and, when Borda scores are distinct, yields the unique minimizer of the sum of Spearman-$\rho$ distances \cite{goibert2023robust} (for more details see Appendix \ref{sec:further_borda}).
\end{example}
 
\begin{example}[Copeland method, \cite{Copeland1951}]
Each permutation $\sigma_v$ is embedded into its pairwise preference matrix, where each pair $(i,j)$ with $i < j$ records $\mathbb{I}\{{\sigma}_v(i) < {\sigma}_v(j)\}$. Averaging over the $N$ voters yields the pairwise preference matrix with entries $\widehat{p}_{ij} =N^{-1}\sum_{v=1}^N \mathbb{I}\{{\sigma}_v(i) < {\sigma}_v(j)\}$. The Copeland score function $S_C$ returns a vector in $\mathbb R^n$ whose $i$-th coordinate is:
\begin{equation}
    \label{eq:copeland_scores}
    \left[S_C\left(\widehat P_N \right)\right]_i = \sum_{j \neq i} \mathbb{I}\{\widehat{p}_{ij} > \tfrac{1}{2}\} - \mathbb{I}\{\widehat{p}_{ij} < \tfrac{1}{2}\},
\end{equation}

counting pairwise wins minus losses. The Copeland consensus ${ \sigma}^C_\star$ ranks candidates in descending order of $[S_C(\widehat P_N )]_i$ (for more details, see Appendix~\ref{sec:further_cope}).
\end{example}
The distribution $P$ is \emph{weakly stochastically transitive} if $p_{ij}\geq 1/2$ and $p_{jk}\geq 1/2$ imply $p_{ik}\geq 1/2$, and \emph{strongly stochastically transitive} if furthermore $p_{ik}\geq \max(p_{ij}, p_{jk})$, for all distinct $(i,j,k)\in\n^3$. Under these conditions, Borda's, Copeland's, and Kemeny's methods yield identical unique solutions (see \cite{korba2017learning}). Outside such favorable cases, Arrow's impossibility theorem establishes that no rank aggregation method can simultaneously satisfy all desirable criteria \cite{arrow1951social}.

\paragraph{Decentralized estimation via gossip algorithms.}
\label{sec:background_gossip}
Let $\mathcal{G} = (V, E)$ be a connected, non-bipartite graph with $V = \llbracket N \rrbracket$. Each agent $v \in V$ holds an initial local state $\mathbf x_v(0)$ lying in a Euclidean space, and the collective goal is to compute the global average $\bar{\mathbf x} = N^{-1}\sum_{v=1}^N \mathbf x_v(0)$ through purely local interactions. Two canonical protocols achieve this. In \emph{synchronous gossip}, all nodes update simultaneously via a doubly stochastic \emph{gossip matrix} $\mathbf W$ ($W_{uv}>0$ only if $(u,v)\in E$ or $u=v$ and $\mathbf W\mathbf 1_N = \mathbf W^\top \mathbf 1_N = \mathbf 1_N$), $\mathbf X(t+1) = \mathbf W\mathbf X(t)$ \cite{tsitsiklis_problems_1984, kempe_gossip-based_2003} with $\mathbf X(t) := [\mathbf x_1 (t)^\top, \dots, \mathbf x_N(t)^\top]^\top $, converging to $\bar{\mathbf x}\mathbf{1}_N^\top$. The rate of convergence depends on the spectral gap of $\mathbf W$ denoted $\lambda_{\mathbf W}$ (see Appendix~\ref{sec:graph_topology} for more details). Larger spectral gaps, achieved by denser or expander-like graphs, yield faster convergence \cite{kempe_gossip-based_2003}. In other words, each node $v\in\N$ receives simultaneously the values of the nodes in its neighborhood denoted $\mathcal N_v$ and updates its state accordingly. In \emph{randomized gossip} \cite{boyd2006randomized}, a single edge $e=(u,v)$ is sampled from $p:=(p_e)_{e\in E}$ at each step, and the two incident nodes average their states while all others remain unchanged. The update of state is thus given by the gossip matrix $\mathbf W_e = \mathbf I_N - (1/2) (\mathbf e_u - \mathbf e_v)(\mathbf e_u - \mathbf e_v)^\top$. The expected update matrix is then $\mathbb{E}\mathbf W_e = \mathbf I_N - (1/2)\mathbf L(p)$, where $\mathbf L(p) = \sum_{e\in E} p_e (\mathbf e_u-\mathbf e_v)(\mathbf e_u-\mathbf e_v)^\top$ is the weighted graph Laplacian. The spectral gap of $\mathbb{E}\mathbf W_e$ is denoted $\lambda(p)/2$ and corresponds to the second smallest eigenvalue of $(1/2)\,\mathbf L(p)$. We refer to \cite{boyd2006randomized, shah2009gossip} for further details (see Appendix~\ref{sec:graph_topology} for more details about graph constructions). In the next section, we connect score-based ranking aggregation to gossip estimation.
\begin{figure}
    \centering
    \begin{subfigure}[b]{0.30\textwidth}
        \centering
              \begin{tikzpicture}[scale=0.85,
          node/.style={circle,draw=black,fill=airforceblue,text=white,
                       font=\small\bfseries,inner sep=3pt,minimum size=20pt},
          lbl/.style={font=\tiny}]
        \node[node] (v1) at (1.8,2.8)  {$v_1$};
        \node[node] (v2) at (0,1.6)    {$v_2$};
        \node[node] (v3) at (0.4,0)    {$v_3$};
        \node[node] (v4) at (2.2,0)    {$v_4$};
        \node[node] (v5) at (3.2,1.4)  {$v_5$};
        \draw[thick,black!50] (v1)--(v2) (v1)--(v5) (v2)--(v3)
                                  (v3)--(v4) (v4)--(v5) (v1)--(v4);
        \node[lbl,above=2pt of v1] {$\phi(\sigma_1)$};
        \node[lbl,above=2pt of v2] {$\phi(\sigma_2)$};
        \node[lbl,below=2pt of v3] {$\phi(\sigma_3)$};
        \node[lbl,below=2pt of v4] {$\phi(\sigma_4)$};
        \node[lbl,below right=2pt of v5] {$\phi(\sigma_5)$};
      \end{tikzpicture}
        \caption{At iteration $t=0$, each node $v$ holds a ranking $\sigma_v\in\Sn$ and embeds it as $\phi(\sigma_v)$.}
    \end{subfigure}
    \hfill
    \begin{subfigure}[b]{0.29\textwidth}
        \centering
              \begin{tikzpicture}[scale=0.85,
          node/.style={circle,draw=black,fill=airforceblue,text=white,
                       font=\small\bfseries,inner sep=3pt,minimum size=20pt},
          lbl/.style={font=\tiny}]
        \node[node] (v1) at (1.8,2.8)  {$v_1$};
        \node[node] (v2) at (0,1.6)    {$v_2$};
        \node[node] (v3) at (0.4,0)    {$v_3$};
        \node[node] (v4) at (2.2,0)    {$v_4$};
        \node[node] (v5) at (3.2,1.4)  {$v_5$};
        \draw[thick,black!50] (v1)--(v2) (v1)--(v5) (v2)--(v3)
                                  (v3)--(v4) (v4)--(v5) (v1)--(v4);
        \node[lbl,above=2pt of v1] {$\frac 1 2 (\phi(\sigma_1) + \phi(\sigma_4))$};
        \node[lbl,above=2pt of v2]  {};
        \node[lbl,below=2pt of v3] {};
        \node[lbl,below=2pt of v4] {$\frac 1 2 (\phi(\sigma_1) + \phi(\sigma_4))$};
        \node[lbl,below right=2pt of v5] {};
        \draw[<->,amber,very thick] (v1) -- node[above,sloped,font=\tiny]{gossip} (v4);
      \end{tikzpicture}
        \caption{At iteration $t=1$, the two connected nodes update their estimates using the average.}
    \end{subfigure}
    \hfill
    \begin{subfigure}[b]{0.30\textwidth}
        \centering
        \begin{tikzpicture}[scale=0.85,
          node/.style={circle,draw=black,fill=airforceblue,text=white,
                       font=\small\bfseries,inner sep=3pt,minimum size=20pt},
          lbl/.style={font=\tiny}]
        \node[node] (v1) at (1.8,2.8)  {$v_1$};
        \node[node] (v2) at (0,1.6)    {$v_2$};
        \node[node] (v3) at (0.4,0)    {$v_3$};
        \node[node] (v4) at (2.2,0)    {$v_4$};
        \node[node] (v5) at (3.2,1.4)  {$v_5$};
        \draw[thick,black!50] (v1)--(v2) (v1)--(v5) (v2)--(v3)
                                  (v3)--(v4) (v4)--(v5) (v1)--(v4);
        \node[lbl,above=2pt of v1] {$\psi(\mathbf x_1)\to \sigma_\star$};
        \node[lbl,above left=2pt of v2]  {};
        \node[lbl,below=2pt of v3] {};
        \node[lbl,below=2pt of v4] {$\psi(\mathbf x_4)\to \sigma_\star$};
        \node[lbl,below =2pt of v5] {};

      \end{tikzpicture}
        \caption{After enough iterations, each node applies $\psi$ to its estimate and ranks the scores to obtain $\sigma_\star$.}
    \end{subfigure}
    \caption{Illustration of randomized $(\phi, \psi)$-gossip on $N=5$ nodes, each holding a ranking, where $\phi$ is an embedding function and $\psi$ an aggregation map.}
    \label{fig:gossip_desc}
\end{figure}

\section{Decentralized algorithms for ranking aggregation}
\label{sec:main_framework}
We consider a system of $N \geq 2$ agents on a graph $\mathcal G$, each holding a preference ranking over $n \geq 1$ alternatives, represented by a permutation $\sigma_v \in \Sn$, for $v=1,\dots,N$, and we denote by $\widehat P_N$ the associated empirical distribution. We consider a pair $(\phi, \psi)$ such that the score of the voting rule can be written as $S(\widehat P_N) = \psi\left(N^{-1} \sum_{v=1}^N \phi(\sigma_v)\right),$ where $I$ is a finite index set with $|I|$ polynomial in $n$ (see Table~\ref{tab:ranking_methods}), $\phi:\Sn \to \mathbb R^I$ is an embedding and $\psi: \mathbb R^I \to \mathbb R^n$ is an aggregation map.
Our framework follows the generalized scoring-rule of \cite{xia_generalized_2008}, up to a normalization factor $N^{-1}$, which can be absorbed by $\psi$. For $\phi$, we define the mixed norm as $\trinorm{\phi} := \max_{\sigma\in \Sn}\|\phi(\sigma) \|$. It can be shown that both Borda and Copeland scores follow this construction \cite{xia_generalized_2008}. In fact, for Borda score, we have $\psi_B= -\mathrm{Id}$ and $\phi_B$ the application that maps a permutation $\sigma$ onto its $\mathbb R^n$ representation, then we find that Borda score is just $S_B(\widehat P_N) = -N^{-1} \sum_{v=1}^N \phi_B(\sigma_v)$. For Copeland score, we consider $\phi_C$ the application that returns the empirical pairwise probabilities and  $[\psi_C((x_{ij})_{i<j})]_i = \sum_{j\ne i} \mathbb I \{x_{ij} >1/2\} - \mathbb I \{x_{ij} <1/2\}$ (see Appendix~\ref{sec:further_cope} for more details). We refer to $S$ as a $(\phi,\psi)$-score function when it admits the decomposition following \cite{xia_generalized_2008}. In what follows, we denote the \emph{average embedding} as $\bar{\mathbf{x}} := \left(\bar{x}^{(1)}, \ldots, \bar{x}^{(|I|)}\right)^{\top} :=N^{-1}\sum_{v=1}^N \phi(\sigma_v)$ and such that $\bar x^{(i)} = N^{-1}\sum_{v=1}^N [\phi(\sigma_v)]_i$. Let $\mathbf{s}:= S(\widehat P_N)$ be the score vector associated with $S$ and $\widehat P_N$, where $s_i = [\psi(\bar {\mathbf x})]_i$. For any $\mathbf s\in\mathbb R^n$, we denote by $[\mathbf s]_{\Sn}\in\Sn$ the permutation induced by argsorting the coordinates of $\mathbf s$ in decreasing order, ties being broken by a fixed arbitrary rule (\textit{e.g.}\ by item index). We now introduce a notion of stability of the induced ranking, which will be key to deriving our convergence rate bounds.

\begin{definition}[Locally ranking-stable]
An aggregation map $\psi: \mathbb R^I \to \mathbb R^n$ is said to be \emph{locally ranking-stable} at $\mathbf{x}\in \mathbb R ^I $ if there exists $r>0$ such that
\[
\forall\, \mathbf y \in \mathbb R^I,\quad \|\mathbf{y}-\mathbf x\|\le r\implies [\psi(\mathbf y)]_{\Sn} = [\psi(\mathbf x)]_{\Sn}.
\]
The \emph{stability radius} of $\psi$ at $\mathbf{x}$ is defined as the supremum of the radii $r\ge 0$.
\end{definition}
A perturbation of $\mathbf x$ that stays within the $\ell_2$-ball of radius $r_\psi(\mathbf x)$ is sufficient, but not necessary, for the ranking to remain unchanged. We note that this quantity is only known asymptotically by each node (see Appendices~\ref{sec:borda_radius} and \ref{sec:cope_radius} for direct applications on Borda's and Copeland's methods). Throughout this section, the dataset $\{\sigma_1,\ldots,\sigma_N\}$ is treated as fixed, so that $\widehat{P}_N$, $\bar{\mathbf{x}}$, the consensus $\sigma_\star$ and $r_\psi(\bar{\mathbf{x}})$ are deterministic quantities. We now describe our two algorithms (Algorithms~\ref{alg:sync-decentralized-cons} and \ref{alg:decentralized-cons}) and refer to Figure~\ref{fig:gossip_desc} for an overview of Algorithm~\ref{alg:decentralized-cons}.

\begin{figure}[ht]
\begin{minipage}[t]{0.54\textwidth}
\begin{algorithm}[H]
\caption{Synchronous $(\phi,\psi)$-Score Gossip.}
\label{alg:sync-decentralized-cons}
\begin{algorithmic}
\STATE \textbf{Init}: Given a coefficient $\alpha\in[1,2]$, a ranking $\sigma_v$, each voter $v$ sets $\mathbf{x}_v = \phi(\sigma_v)\in\mathbb R^I$.
\FOR{$t\in \llbracket T \rrbracket $}
\FOR{$v\in\N$ in parallel}
\FOR{$u$ in neighborhood of $v$}
\STATE Send $\mathbf x_v$, receive $\mathbf x_u$.
\ENDFOR
\STATE \textbf{if} $t=1$ \textbf{then} $\mathbf{x}_v, \mathbf x'_v \gets \sum_{u} W_{vu} \mathbf x_u,\, \mathbf x _v$.
\STATE \textbf{else} $\mathbf{x}_v, \mathbf x'_v \gets (1-\alpha)\mathbf{x}'_v +\alpha\sum_{u} W_{vu} \mathbf x_u,\, \mathbf x _v$.
\ENDFOR
\ENDFOR
\STATE Each voter $v$ computes $\psi(\mathbf x_v)$ locally. 
\STATE \textbf{Output:} Each voter $v$ returns the ranking associated with $\psi(\mathbf x_v)$.
\end{algorithmic}
\end{algorithm}
\end{minipage}
\hfill
\begin{minipage}[t]{0.42\textwidth}
\begin{algorithm}[H]
\caption{Randomized Asynchronous $(\phi,\psi)$-Score Gossip.}
\label{alg:decentralized-cons}
\begin{algorithmic}
\STATE \textbf{Init}: Given a probability distribution $ p = (p_e)_{e\in E}$, a ranking $\sigma_v$, each voter $v$ sets $\mathbf{x}_v = \phi(\sigma_v)\in\mathbb R^I$.\\
\vspace{1em}
\FOR{$t\in \llbracket T \rrbracket $}
\STATE Select an edge $e = (u,v)\in E$ with probability $p_e>0$.
\STATE Update $\mathbf{x}_u, \mathbf{x}_v \leftarrow ({\mathbf{x}_u + \mathbf{x}_v})/2$.
\ENDFOR
\vspace{1.25em}
\STATE Each voter $v$ computes $\psi(\mathbf x_v)$ locally.
\vspace{0.1em}
\STATE \textbf{Output:} Each voter $v$ returns the ranking associated with $\psi(\mathbf x_v)$.
\end{algorithmic}
\end{algorithm}
\end{minipage}
\end{figure}

\paragraph{Synchronous score gossip.} The synchronous gossip algorithm for $(\phi, \psi)$-scores is given in Algorithm~\ref{alg:sync-decentralized-cons}. At each round $t\ge 1$, every node $v \in \N$ simultaneously exchanges its current estimate $\mathbf{x}_v(t)$ with all its neighbors and performs a weighted average update governed by the gossip matrix $\mathbf W$. After $T$ rounds, each node independently decodes its local estimate into a ranking via $\psi$.

The convergence rate is governed by the spectral gap $\lambda_{\mathbf W}$: larger values of $\lambda_{\mathbf W}$ yield faster concentration of the local estimates $\mathbf{x}_v(t)$ around the average embedding $\bar{\mathbf{x}}$. Since $\psi$ is assumed to be locally ranking-stable at $\bar{\mathbf x}$, any estimate that is sufficiently accurate recovers the consensus ranking $\sigma_\star$. Introducing an auxiliary iterate $\mathbf x'_v(t)$ yields the accelerated second-order scheme of \cite{cao_accelerated_2006}, whose rate scales with $\sqrt{\lambda_{\mathbf W}}$. This notably involves an expression of the update matrix as a Chebyshev polynomial (see \cite{liu_analysis_2013, berthier_accelerated_2019}), which gives the optimal memory coefficient $\alpha_\star = (2 - 2\sqrt {\lambda_{\mathbf W}(1 - \lambda_{\mathbf W}/4)})/ (1- \lambda_{\mathbf W}/2)^2$. When $\alpha = 1$, the second estimate $\mathbf{x}'_v$ is unused and the scheme reduces to the standard first-order synchronous gossip protocol, in which each node maintains only a single embedding vector $\mathbf{x}_v$. Our convergence guarantee is stated in Theorem~\ref{thm:sync-general-ranking-convergence}, and its proof is given in Appendix~\ref{sec:proofs}.
\begin{theorem}
\label{thm:sync-general-ranking-convergence}
Consider Algorithm~\ref{alg:sync-decentralized-cons} with $\alpha = \alpha_\star$. Let $\sigma_\star$ be the consensus ranking for $\widehat P_N$, and $\mathbf W$ be the gossip matrix, with absolute spectral gap $\lambda_{\mathbf W}>0$. Assume $\psi$ satisfies local ranking-stability at the average embedding $\bar{\mathbf{x}}$. For each $i \in I$, we denote the estimate of the $i$-th coordinate across all voters at time $t$ as $\mathbf{x}^{(i)}(t) = (x_{1i}(t), x_{2i}(t), \ldots, x_{Ni}(t))^{\top}\in\mathbb R^N$. Then, the average expected Kendall-$\tau$ error satisfies for $t\ge 0$:
\begin{equation}\label{eq:sync_conv}
\frac{1}{N}\sum_{v=1}^N d_\tau\left(\sigma_\star, \widehat{\sigma}_v(t)\right)\leq \frac{2\binom n 2 }{Nr_\psi(\bar{\mathbf x})^2} e^{-\sqrt {\lambda_{\mathbf W}}  t}\underbrace{\sum_{i\in I} \|\mathbf x^{(i)} (0) - \bar{x}^{(i)} \mathbf 1 _N \|^2}_{=:\gamma},
\end{equation}
where $\widehat{\sigma}_v(t):=[\psi(\mathbf{x}_v(t))]_{\Sn}$, \textit{i.e.,} the ranking induced by the scores computed by node $v$ at time $t$.
\end{theorem}
We call $\gamma$ the \emph{initial dispersion} of the embeddings; it is deterministic since the dataset is fixed. The acceleration comes at a memory cost: each node $v\in\N$ must store an additional embedding vector $\mathbf x'_v$. Without this additional vector, we recover the standard synchronous gossip update. The convergence is thus in $\mathcal O (e^{-\lambda_{\mathbf W}t})$ when $t\to\infty$, which is slower as $\lambda_{\mathbf W}< 1$ due to the doubly stochastic nature of $\mathbf W$. Theorem~\ref{thm:sync-general-ranking-convergence} also yields a bound on the number of rounds needed for all nodes to recover the consensus ranking.
\begin{remark}[$\delta$-mixing time]\label{rem:sync-mixing-time}
We can upper bound the \emph{$\delta$-mixing time} defined for any $\delta\ge 0 $ as $T^\mathrm{Sync}(\delta) = \min\{t\ge 0 \mid N^{-1}  \sum_{v=1}^N \,d_\tau (\widehat \sigma_v(t), \sigma_\star)\le \delta\}.$ By Theorem~\ref{thm:sync-general-ranking-convergence}, the mixing time needed to recover the consensus ranking ($\delta=0$) satisfies $T^\mathrm{Sync}(0) \le \max\{0,\lfloor t^\star\rfloor+1\}$, where we define $t^\star:=\log\left(2\binom n 2 \gamma /r_\psi(\bar{\mathbf x})^2\right)/ \sqrt{\lambda_{\mathbf W}}$. 
\end{remark}
\paragraph{Randomized score gossip (asynchronous).} The randomized gossip algorithm for $(\phi, \psi)$-scores is given in Algorithm~\ref{alg:decentralized-cons}. All expectations are taken with respect to the random edge sampling only. The asynchronous gossip process is a Markov chain with a fixed initial state. By relating it to standard randomized gossip averaging \cite{boyd2006randomized}, we obtain the following convergence result (see Appendix~\ref{sec:conv_gossip} for further properties of randomized gossip processes). The proofs of Theorem~\ref{thm:general-ranking-convergence} and of the remarks of this section are given in Appendix~\ref{sec:proofs}.
\begin{theorem}
\label{thm:general-ranking-convergence}
Consider Algorithm~\ref{alg:decentralized-cons} with edge probabilities $p=(p_e)_{e\in E}$ such that $\lambda(p)>0$. Assume $\psi$ satisfies local ranking-stability at $\bar{\mathbf{x}}$. Then, the average expected Kendall-$\tau$ error satisfies, for all $t\ge 0$,
\begin{equation}
    \frac{1}{N}\sum_{v=1}^N \mathbb E \,d_\tau\left(\sigma_\star, \widehat{\sigma}_v(t)\right)\leq \frac{\binom n 2 }{Nr_\psi(\bar{\mathbf x})^2} e^{-\lambda(p) t/2}\gamma,
\end{equation}
where $\widehat{\sigma}_v(t):=[\psi(\mathbf{x}_v(t))]_{\Sn}$ is the ranking induced by the scores computed by node $v$ at time $t$, and $\gamma$ and $\lambda(p)$ are defined in Equation~\eqref{eq:sync_conv} and Section~\ref{sec:background_gossip}, respectively.
\end{theorem}

Theorem~\ref{thm:general-ranking-convergence} shows that each node recovers the consensus ranking $\sigma_\star$ asymptotically, at an exponential rate in expectation. Note that this bound is not tight, and we observe faster convergence in practice (see Section~\ref{sec:numerical_exp}). The convergence of both algorithms, as well as the stability radius, admits a simple geometric interpretation.

\begin{remark}[Geometric interpretation]\label{rem:geom_interp}
The map $[\psi(\cdot)]_{\Sn}$ partitions the convex hull of the initial embeddings $(\mathbf x_v(0))_{v=1}^N$ into ranking cells $C(\sigma) = \{\mathbf x\in\mathbb R^I\mid [\psi(\mathbf x)]_{\Sn}=\sigma\}$, each consisting of the points $\mathbf x$ whose induced ranking is $\sigma\in \Sn$. The ranking consensus is achieved once every node embedding belongs to $ C(\sigma_\star)$, without requiring exact averaging in $\mathbb R^I$. Moreover, if a perturbation occurs during the gossip protocol but the average embedding vector remains inside the cell $\mathcal C(\sigma_\star)$, then the consensus ranking is unchanged.
The stability radius $r_\psi(\bar{\mathbf x})$ is then the distance from $\bar{\mathbf x}$ to the boundary of the corresponding ranking cell $ C(\sigma_\star)$.
\end{remark}

As in the synchronous case, Theorem~\ref{thm:general-ranking-convergence} also yields a bound on the number of iterations needed to recover the consensus ranking.

\begin{remark}[Expected $\delta$-mixing time]\label{rem:mixing-time}
    Similarly to the synchronous case (Remark~\ref{rem:sync-mixing-time}), we define the expected $\delta$-mixing time, for any $\delta>0$, as
    \[
    T^\mathrm{rand}(\delta) = \min\left\{t\ge 0 \,\middle|\, \frac 1 N  \sum_{v=1}^N \mathbb E\, d_\tau (\widehat \sigma_v(t), \sigma_\star)\le \delta\right\}.
    \]
    Theorem~\ref{thm:general-ranking-convergence} yields the following upper bound: $T^\mathrm{rand}(\delta) \le \max\{0,\lfloor t^\star_{\mathrm{rand}}\rfloor+1\}$, where we define $t^\star_{\mathrm{rand}}:=(2/\lambda(p)) \log\left(\binom n 2 \gamma/(\delta r_\psi(\bar{\mathbf x})^2)\right)$. Hence, the mixing time of the randomized process scales as $2/\lambda(p)$, whereas it scales as $1/\sqrt{\lambda_{\mathbf{W}}}$ in the synchronous case. Moreover, $T^{\mathrm{rand}}(\delta)$ is lower bounded as $T^{\mathrm{rand}}(\delta) \ge \max\{0, (\sum_{v=1}^N d_\tau (\widehat \sigma_v(0), \sigma_\star)- N\delta)/(n(n-1))\}$.
\end{remark}

In practice, voters often provide partial rather than complete rankings of alternatives. Appendix~\ref{sec:practical-tie} provides convergence guarantees for this more general setting. Appendices~\ref{sec:further_borda} and~\ref{sec:further_cope} specialize our results to the Borda and Copeland rules, and Appendix~\ref{sec:gossip-other} extends the framework to metric-based consensus under the Kendall-$\tau$ and Spearman footrule distances. Note that consensus under the Spearman-$\rho$ distance $d_\rho$ reduces to the Borda case. We now study the robustness of our gossip algorithms.

\section{Robustness guarantees using empirical breakdown functions }
\label{sec:robustness}
\label{sec:adversarial_model} We consider a \emph{weak adversarial model} in which a subset $B\subset\N$ of $b:=|B|$ nodes is corrupted and may report arbitrary rankings. The remaining nodes $H := \N \setminus B$ constitute the \emph{honest subgraph} $\mathcal{G}_H$, defined as the induced subgraph of $\mathcal{G}$ with vertex set $H$ and edge set $E \cap H^2$. In particular, following vote manipulation theory \cite{lu_bayesian_2012}, we consider that each corrupted node holds an adversarial initial ranking while still participating in the consensus process. We assume $\mathcal{G}_H$ remains connected, which ensures that honest nodes can still reach consensus among themselves through gossip averaging \cite{su_byzantine_2015, gaucher_unified_2025}. An adversary controlling $b$ nodes can shift the final consensus regardless of the communication graph. Although our adversarial model assumes corrupted nodes follow the averaging protocol, the geometric structure of the embedding provides resilience against stronger Byzantine attacks in which corrupted nodes inject arbitrary continuous values during the averaging phase. Since $\phi(\sigma) \in \mathcal{P}_\phi$ for all $\sigma \in \Sn$, the polytope $\mathcal{P}_\phi$ is compact and bounded. Consequently, any Byzantine perturbation not in $\mathcal{P}_\phi$ injected during averaging is implicitly clipped by $\psi$. The contamination influence on the output ranking is bounded by the stability radius $r_\psi(\bar{\mathbf{x}})$ characterized in Remark~\ref{rem:geom_interp}, regardless of the magnitude of the injected values.

We define the minimal contamination fraction needed to modify the estimated score $S(\widehat P_N)$, a concept originally introduced in the robust statistics literature \cite{huber_robust_2009, hampel_general_1971}, subsequently adapted to gossip learning in \cite{van2026robust} and to ranking distributions in \cite{goibert2022statistical, goibert2023robust}, as follows. 

\begin{definition}[Empirical Breakdown Function] 
The \emph{Empirical Breakdown Function} is the minimum fraction of corrupted votes required to shift the consensus by a distance of at least $\delta> 0$:

\begin{equation}
    \varepsilon_{\widehat P_N, S}^{\star}(\delta)=\inf \left\{\frac{b}{N} \;\middle|\; b \in \N, \sup _{\widehat P_N^b \in \mathcal{B}\left(\widehat P_N, b/N\right)} d_\tau\left(S\left(\widehat P_N\right), S\left(\widehat P_N^b\right)\right) \geq \delta\right\},
\end{equation}

where $\mathcal{B}(\widehat P_N, \varepsilon) \subset \Delta_{\Sn}$ denotes the set of empirical distributions $\widehat P_N^b$ that can be obtained by modifying $b =  \lceil\varepsilon N\rceil$ of the $N$ votes. Formally, 
\[\mathcal{B}(\widehat P_N,  b) = \left\{\widehat P_N^b \in \Delta_{\Sn} \;\middle|\; d_\mathrm{TV}(\widehat P_N, \widehat P_N^b) \leq  \frac b N, \, \widehat P_N^b = \frac 1 N\sum_{v=1}^N \delta_{\tau_v}, \{\tau_1, \dots,\tau_N\} \in \mathrm{Sym}^N (\Sn)\right\},\]
where $d_\mathrm{TV}$ denotes the total variation distance: $d_\mathrm{TV}(P,Q) =  (1/2)\sum_{\sigma\in\Sn}|P(\sigma) - Q(\sigma)|$. 
\end{definition}
\begin{remark}\label{rem:MoV}
    The concept of Margin of Victory (MoV) \cite{reisch_margin_2014} is related to the empirical breakdown. Indeed, the MoV is defined as the minimal number of votes required to change the top alternative, while the empirical breakdown measures how the full consensus ranking changes with respect to the Kendall-$\tau$ metric $d_\tau$. Furthermore, computing the MoV is NP-complete for pair-based voting rules such as Copeland and maximin, while it is computable in polynomial time for methods such as Bucklin or Borda \cite{xia_computing_2012}, which leads us to believe that computing the empirical breakdown function in the general case is also NP-complete.
\end{remark}
\begin{proposition}
\label{prop:bound_breakdown}
Let $\delta>0$. For a given distribution $\widehat P_N$ and score function $S$ with $(\phi,\psi)$-decomposition, the breakdown function admits the following lower bound for  $d_\tau$:
\begin{equation}
    \varepsilon ^-  :=\frac{r_\psi(\bar{\mathbf x})}{2 \trinorm{\phi}}\le \varepsilon^\star_{\widehat P_N, S}(\delta).
\end{equation}

\end{proposition}
Proposition~\ref{prop:bound_breakdown} provides a general characterization of robustness for any score function admitting a $(\phi, \psi)$-decomposition.  The lower bound $\varepsilon^-$ quantifies the minimum fraction of corrupted votes an adversary must inject to shift the consensus. In other words, as long as the contamination fraction remains below this threshold, the consensus ranking is unchanged. The limiting factor in Proposition~\ref{prop:bound_breakdown} is the geometry of $\psi$. The lower bound $\varepsilon^-$ is independent of $\delta$, reflecting the fact that once the threshold is passed, then the contamination can corrupt with an error to consensus $\delta>1$. We note that $\varepsilon^- = \mathcal O(n^{-\beta})$ where $\beta$ is a constant depending on $\phi$ outputs (see Appendix~\ref{sec:prob_model}).

\paragraph{Time-dependent breakdown in gossip protocols.}
Consider a gossip process as described in Algorithm~\ref{alg:decentralized-cons}. Let $\sigma_\star \in \Sn$ denote the target consensus reached by honest nodes when $B = \emptyset$, \textit{i.e.}, $\hat{\sigma}_v(t)\to\sigma_\star$ when $t\to\infty$ for any $v \in H$. Theorem~\ref{thm:timed_break} establishes that breakdown resilience improves exponentially with the number of iterations. We consider the \emph{average} breakdown point over honest nodes:
\[\varepsilon_t^{H}(\delta) := \inf\left\{\frac{b}{N} \;\middle|\; \sup_{\widehat{P}_N^b \in \mathcal{B}(\widehat{P}_N,\, b)} \mathbb{E}\,d_\tau\left(\psi\left(\frac{1}{N-b}\sum_{v\in H}\mathbf{x}_v(t)\right), \sigma_\star\right) \ge \delta\right\}.\]
If $B = \emptyset$, the Kendall-$\tau$ distance of the honest average with the true consensus $\sigma_\star$ is always null. The following theorem quantifies this convergence for $B \neq \emptyset$, as $t \to \infty$, the bound converges to the weaker bound $\varepsilon^- \delta/\binom{n}{2}$.
\begin{theorem}\label{thm:timed_break}
    Let $\lambda_{1} (\bar {\mathbf W})$, $\lambda_{\bar {\mathbf W}}$ be the smallest eigenvalue and the spectral gap of $\bar {\mathbf W}$, respectively. Then, the empirical breakdown point at iteration $t\ge 0$ is lower bounded by:
\begin{equation}\label{eq:avg_timed_eps}
\varepsilon_t^H \ge \frac {\left(\sqrt{\xi_t^2+4\varepsilon^- \delta (1+\zeta_t)/\binom n 2}-\xi_t\right)^2}{4(1+\zeta_t)^2},
\end{equation}
where $\xi_t =(1- \lambda(p))^{t/2}$ and $\zeta_t = 1 - \lambda_1^t(\bar{\mathbf W})$.
\end{theorem}
The bound in \eqref{eq:avg_timed_eps} reflects two competing spectral phenomena. The term $\xi_t$ controls the residual disagreement among honest nodes (how far their average is from mixing), while $\zeta_t$ captures the accumulated contamination from adversarial nodes, growing as their influence percolates through the graph. 

\section{Numerical experiments}
\label{sec:numerical_exp}
In this section, we evaluate the proposed decentralized consensus methods through a series of numerical experiments\footnote{Code for all experiments is publicly available at https://github.com/kerrianlc/decentralized-ranking-aggregation.}. We consider two classes of datasets: synthetic data generated from the Mallows model, and real-world preference data \cite{MaWa13a}, specifically the Sushi \cite{sushi} and Debian's 2005 leader election datasets. In all settings, the gossip communication follows a uniform distribution over edges, \textit{i.e.}, $p = (1/|E|)_{e \in E}$. Additional experimental setups and supplementary results are provided in Appendix~\ref{sec:other_exp}.
\subsection{Gossip convergence experiments}
\label{sec:gossip_conv_exp}
We evaluate our methods on three datasets. Mixture of Mallows (see Appendix~\ref{sec:prob_model} for the definition), $N=1000$ synthetic rankings over $n=10$ alternatives drawn from a two-component mixture with weights $(0.7, 0.3)$, dispersion parameters $(\varphi=0.4, \varphi=0.6)$, and reference rankings $\sigma_\star$ and $\sigma_{\mathrm{rev}}$ which is $\sigma_\star$ in reversed order, respectively; we run $500$ trials over $20000$ iterations. Sushi: strict complete rankings of $n=10$ alternatives by $N=5000$ voters; $25$ trials over $200000$ iterations. Debian: partial rankings of $n=7$ alternatives by $N=504$ voters; $250$ trials over $50000$ iterations. Sushi and Debian are taken from PrefLib~\cite{MaWa13a}. All experiments use four graph topologies (complete, 2D grid, geometric, Watts-Strogatz), with data reshuffled across nodes per trial. To handle partial rankings, we apply rank normalization to infer the Borda consensus and the corresponding weak pairwise preferences to deduce the Copeland consensus, as detailed in Appendix~\ref{sec:practical-tie}.\\

\begin{figure*}[ht]
    \centering

    \begin{subfigure}{0.30\textwidth}
        \centering
        \includegraphics[height=\myheight, width=\mywidth]{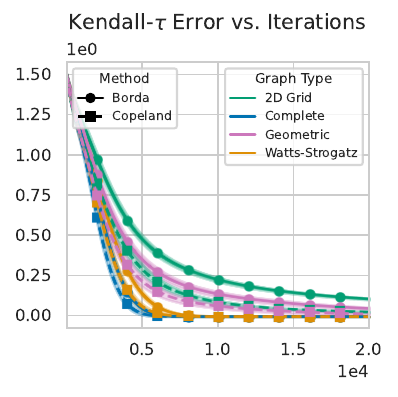}
        \caption{Mixture of Mallows}
    \end{subfigure}
    \begin{subfigure}{0.27\textwidth}
        \centering
        \includegraphics[height=\myheight, width=\mywidth]{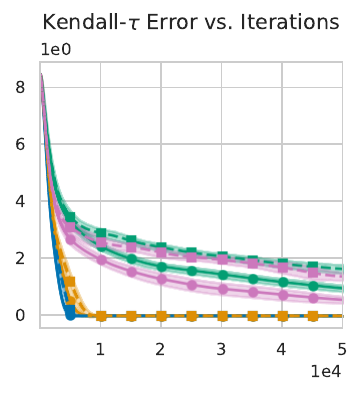}
        \caption{Debian dataset}
    \end{subfigure}
    \begin{subfigure}{0.30\textwidth}
        \centering
        \includegraphics[height=\myheight, width=\mywidth]{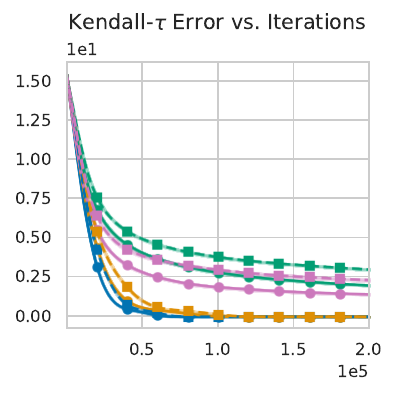}
        \caption{Sushi dataset}
    \end{subfigure}
    \caption{
    Average local Kendall-$\tau$ error to the respective consensus under randomized gossip.
    }

    \label{fig:convergence_both_datasets}
\end{figure*}

 In Figure~\ref{fig:convergence_both_datasets}, we present the convergence behavior of the decentralized Borda and Copeland methods. The plots show the Kendall-$\tau$ distance to the ground-truth consensus averaged over the trials with $\pm2$ standard error. For both the Sushi and Debian datasets, we operate in a favorable setting in which both methods (Borda and Copeland) yield identical consensus rankings. For both gossip-based Borda and Copeland, convergence is exponential and consistent with the underlying graph topology, slower on the 2D grid and geometric graphs, and faster on the Watts-Strogatz and complete graphs (see Appendix~\ref{sec:graph_topology} for more details). For all experiments, the errors of the two consensus algorithms decrease at very similar rates and are consistent with Theorem~\ref{thm:general-ranking-convergence}. This highlights that convergence speed is governed by the spectral gap of the graph and the initial data distribution rather than the choice of voting rule.

\subsection{Breakdown point study}
\label{sec:breakdown_expe}
We consider $N = 3000$ agents and $n = 7$ alternatives over a Watts-Strogatz communication graph. The honest fraction $1-\varepsilon$ of agents holds rankings drawn i.i.d.\ from a Mallows model $P^M_{\theta}$ with parameter $\theta = (\sigma_\star, \varphi) \in \Sn \times [0,1]$, where $\varphi = 0.6$. The contaminated fraction $\varepsilon$ of agents holds $\sigma_{\mathrm{rev}}$, the maximally reversed ranking, which is the worst-case adversarial perturbation. We vary the contamination fraction $\varepsilon \in [0.01, 0.45]$ across $25$ values and run $1000$ independent trials per configuration. For each trial, we run the gossip algorithms for $t \in \{5000, 10000, 20000\}$ iterations and measure the Kendall-$\tau$ distance $\delta$ between the honest consensus and the true consensus $\sigma_\star$. 
We estimate the empirical breakdown curves by running $1000$ trials with a different contaminated subset at each trial. This methodology is motivated by the computational complexity discussed in Remark~\ref{rem:MoV}.

Figure~\ref{fig:breakdown} reports the empirical breakdown curves. The solid step curve is a worst-case statistic. For each $\delta\ge 0$ on the $x$-axis, it marks the smallest contamination $\varepsilon$ at which at least one of the $1000$ trials produced an error greater than $\delta$.

First, at low iteration counts ($t=5000$), the breakdown curve exhibits a pronounced staircase structure with discrete jumps as contamination crosses critical thresholds, which smooths into a plateau by $t=10000$ as the residual spectral term $\xi_t$ in Theorem~\ref{thm:timed_break} decays. Second, Copeland exhibits an initial fragility, breaking at lower contamination levels than Borda in early iterations, but shows lower Kendall-$\tau$ error as contamination percolates through the network. Both methods achieve breakdown thresholds substantially higher than the worst-case bound $\varepsilon^-$ from Proposition~\ref{prop:bound_breakdown}, reflecting the gap between theoretical guarantees and average-case robustness. 
Third, breakdown thresholds slightly decrease with $t$ from $t=10000$ to $t=20000$, reflecting contamination propagation through the network as adversarial rankings diffuse into the honest average, though the degradation rate indicates that practical deployments reach near-stationary robustness within finite runtime.
\begin{figure*}[ht]
    \centering
    \begin{subfigure}{0.30\textwidth}
        \centering
        \includegraphics[height=\contaheight]{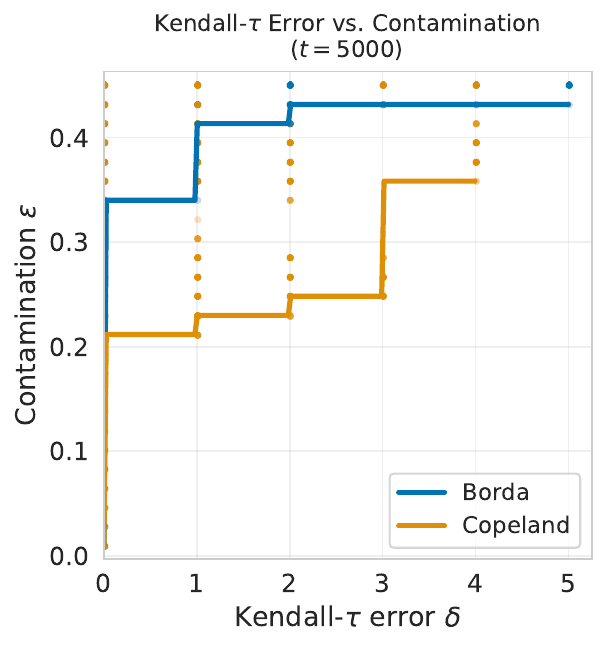}
        \caption{At iteration $t=5000$}
    \end{subfigure}
    \begin{subfigure}{0.30\textwidth}
        \centering\includegraphics[height=\contaheight]{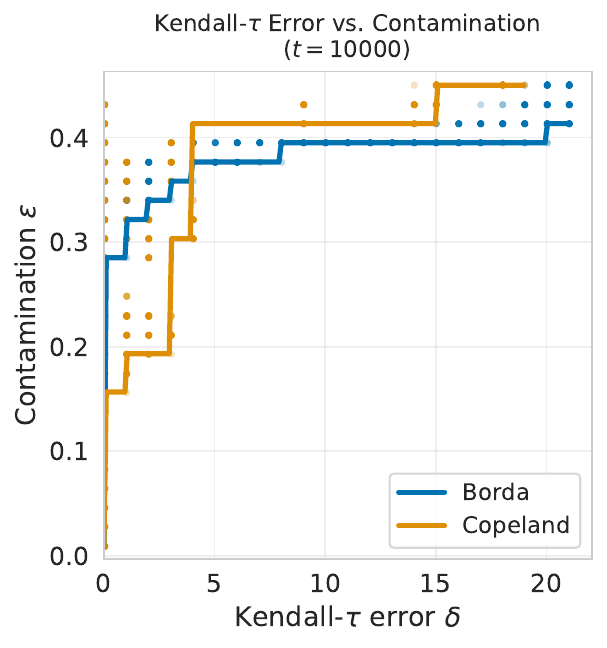}
        \caption{At iteration $t=10000$}
    \end{subfigure}
    \begin{subfigure}{0.30\textwidth}
        \centering\includegraphics[height=\contaheight]{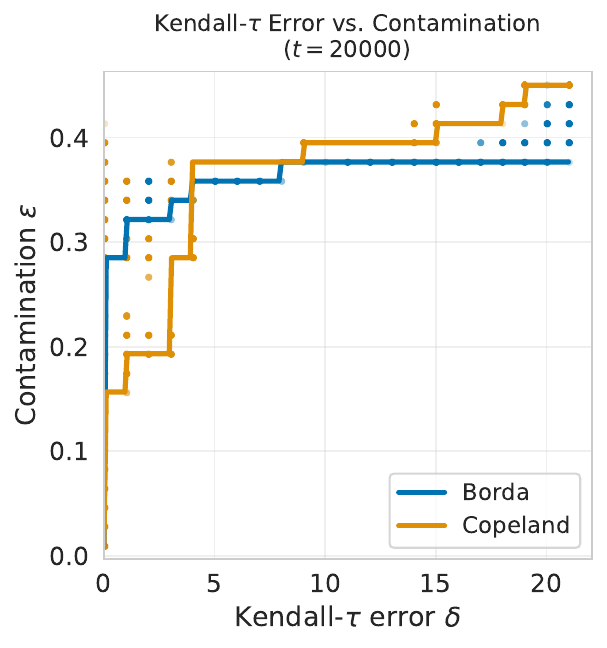}
        \caption{At iteration $t=20000$}
    \end{subfigure}\hfill
    \caption{Average Kendall-$\tau$ error vs. contamination for the Mallows distribution.}
    \label{fig:breakdown}
\end{figure*}

\subsection{Application to inference tasks} We consider the following setup. Each node $u\in \N$ on the communication graph represents a learner which holds a private local dataset $D_u$ on which a model is trained. Then, during an inference phase, each learner share their prediction on a common query and find a consensus on the shared prediction using a communication protocol. This framework is referred to as \emph{Decentralized ensemble learning}~\cite{chen_fedbe_2020}, where the goal is to combine the predictions of multiple models to improve overall performance. In our case, we assume the output to be a ranking. 

We consider two experiments for inference tasks. First, we consider a linear model $y = \langle \boldsymbol\theta, \mathbf d\rangle + \epsilon$ with $\boldsymbol\theta\in\mathbb R^d$ $(d=4)$ and $\epsilon\sim\mathcal N(0,\varsigma_u^2)$, where the noise level $\varsigma_u$ specific to node $u\in\N$  $(N=100,n=6)$ sampled following a uniform distribution on $[0.1, 0.4]$. Each node $u$ holds a private training set of i.i.d.\ features $\mathbf d$ with their noisy labels, and computes the least-squares estimate $\hat{\boldsymbol\theta}_u$. All nodes share a test set $\mathbf d_1,\dots,\mathbf d_n\in\mathbb R^d$ which allow to make a test score vector $\mathbf s_u = ( \langle \hat {\boldsymbol\theta}_u, \mathbf d_1\rangle , \dots ,\langle \hat {\boldsymbol \theta}_u, \mathbf d_n\rangle )$ is then argsorted to obtain the ranking $\sigma_u\in\Sn$. The true scores $\langle\boldsymbol\theta,\mathbf d_j\rangle$ are pairwise distinct and define similarly the ground-truth ranking $\sigma^\star\in\mathfrak S_n$. We consider a second setting motivated by RLHF-style pipelines, we consider a rating dataset, each node fits a local Plackett-Luce \cite{plackett_analysis_1975, luce_individual_1959} utility model from its preferences and returns a score which can be converted into a (partial) ranking. We use the 'MovieLens small' dataset \cite{harper_movielens_2015} as a proof of concept of the pipeline, where we selected $N=400$, $n=25$ of the most contributing users, and most rated movies. Figure~\ref{fig:convergence_movielens} illustrates the convergence of the previous tasks.
\begin{figure*}[ht]
    \centering

    \begin{subfigure}{0.30\textwidth}
        \centering
        \includegraphics[height=\myheight, width=\mywidth]{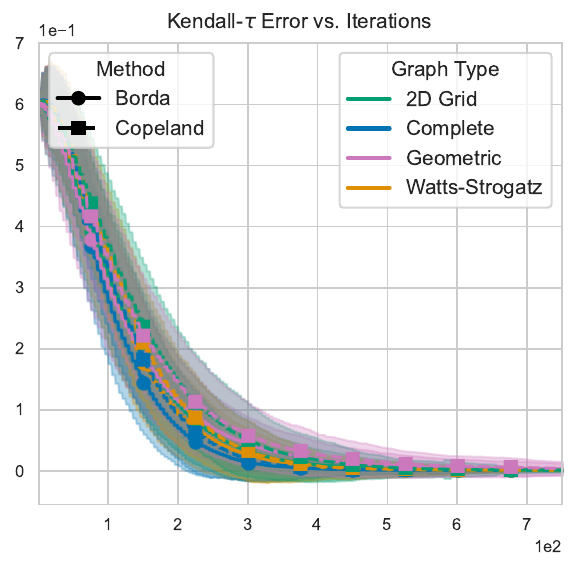}
        \caption{Toy model}
    \end{subfigure}
    \begin{subfigure}{0.30\textwidth}
        \centering
        \includegraphics[height=\myheight, width=\mywidth]{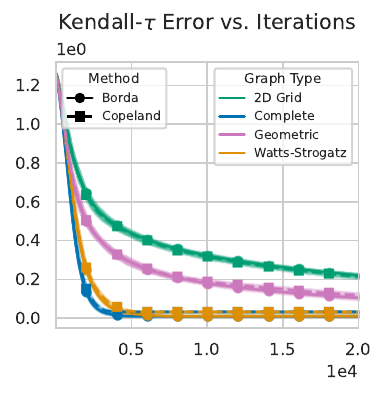}
        \caption{MovieLens}
    \end{subfigure}
    \caption{Convergence of inference tasks under randomized gossip.
    }
\label{fig:convergence_movielens}
\end{figure*}

\section{Conclusion and perspectives}
\label{sec:conclusion}
We have introduced a family of gossip-based algorithms for decentralized ranking aggregation. Our theoretical analysis establishes exponential convergence rates for score-based consensus. 
\paragraph{Limitations.}
Our framework rests on several assumptions worth highlighting. Convergence requires $\mathcal G$ to remain connected, and per-round communication costs scale as $O(n^2)$ for Copeland, which may be prohibitive for large $n$. The adversarial model of Section~\ref{sec:robustness} further assumes corrupted nodes follow the averaging protocol, ruling out more sophisticated attacks such as delayed, selective, or adaptive message manipulation. The upper bounds considered in the convergence analysis are often loose. In practice, the process may exhibit a faster decay as suggested by some empirical observations.
\paragraph{Perspectives.} A natural extension is the design of provably robust gossip protocols under structured adversarial models \cite{he_byzantine-robust_2023}. Decentralized settings also require careful consideration of neighbor trustworthiness.  Algorithms~\ref{alg:sync-decentralized-cons} and \ref{alg:decentralized-cons} provide no privacy protection, because $\phi$ is injective. Indeed, since the embedding $\phi$ maps $\Sn$ injectively onto the finite set $\{\phi(\sigma) \mid \sigma \in \Sn\}$, the first message from node $v$ to any neighbor reveals $\sigma_v$ exactly. A solution is to compose these algorithms with a differential-privacy mechanism~\cite{cyffers_muffliato_2022} that injects noise into each exchanged embedding. The variance must be sufficiently small to ensure that the ranking is preserved with high probability.

\begin{ack}
    This work benefited from Hi! PARIS and State funding managed by the French National Research Agency (ANR) under the France 2030 program, reference ANR-23-IACL-0005. It was also supported by the PEPR IA Foundry grant.
\end{ack}

\bibliographystyle{plainnat}
\bibliography{reference}

\newpage
\appendix
\newcounter{appthm}[section]
\renewcommand{\theappthm}{\thesection.\arabic{appthm}}
\let\definition\relax\let\enddefinition\relax\newtheorem{definition}[appthm]{Definition}
\let\lemma\relax\let\endlemma\relax\newtheorem{lemma}[appthm]{Lemma}
\let\assertion\relax\let\endassertion\relax\newtheorem{assertion}[appthm]{Assertion}
\let\proposition\relax\let\endproposition\relax\newtheorem{proposition}[appthm]{Proposition}
\let\theorem\relax\let\endtheorem\relax\newtheorem{theorem}[appthm]{Theorem}
\let\corollary\relax\let\endcorollary\relax\newtheorem{corollary}[appthm]{Corollary}
\let\example\relax\let\endexample\relax\newtheorem{example}[appthm]{Example}
\let\remark\relax\let\endremark\relax\newtheorem{remark}[appthm]{Remark}
\section*{Outline of the Supplementary Materials}
The supplementary material is organized as follows. Appendix~\ref{sec:graph_topology} describes the graph topologies used throughout our experiments. Appendix~\ref{sec:conv_gossip} details the randomized gossip communication protocols and some properties of these protocols. Appendix~\ref{sec:two_layers} presents the two-layer aggregation framework for score-based aggregation. Appendix~\ref{sec:proofs} contains all proofs of the main theoretical results stated in the paper. Appendix~\ref{sec:practical-tie} discusses extensions of our framework to broader settings especially considering partial rankings. Appendices~\ref{sec:further_borda} and \ref{sec:further_cope} provide further analysis of the Borda and Copeland voting rules under the developed framework. Appendix~\ref{sec:prob_model} formalizes and discusses the probabilistic model considerations for data. Appendix~\ref{sec:gossip-other} describes additional baselines notably for metric-based consensus and their comparisons and robustness. Finally, Appendix~\ref{sec:other_exp} presents additional experiments and ablation studies, while Appendix~\ref{sec:compute} details the compute resources used for the different experiments of the paper.

\section{Graph requirements and topologies}
\label{sec:graph_topology}
\subsection{Spectral properties of gossip matrices}

The convergence of gossip-based averaging algorithms is tightly linked to the spectral properties of the underlying communication graph $\mathcal{G} = (\N, E)$. We briefly recall the relevant definitions and describe some key spectral quantities. In gossip protocols, transitions from one state to another are given by a gossip matrix whose definition is recalled here.

\begin{definition}[Gossip Matrix]
    A matrix $\mathbf W$ is said to be a \emph{gossip matrix} if and only if:
    \begin{itemize}
    \item $\mathbf W$ is symmetric, \textit{i.e.,} $W_{uv} =  W_{vu}$ for all $u\ne v$.
    \item $W_{uv} > 0$ when $(u,v) \in E $ and $W_{uv} = 0$ when $(u,v)\notin E$.
    \item $W_{vv}\ge 0$ for all $v\in \N$.
        \item $\mathbf W$ is doubly stochastic, \textit{i.e.,} $\mathbf W\mathbf 1_N = \mathbf W^\top\mathbf 1_N = \mathbf 1_N$.
    \end{itemize}
\end{definition}

A necessary structural requirement is that $\mathcal{G}$ remains connected. For a symmetric real matrix $\mathbf{W}$, we denote its spectrum by $\operatorname{Sp}(\mathbf{W})$, with eigenvalues ordered as $\lambda_1(\mathbf{W}) \le \cdots \le \lambda_N(\mathbf{W})$. Since $\mathcal{G}$ is connected, the Perron-Frobenius theorem ensures that the eigenvalue $1$ of $\mathbf{W}$ is simple and its associated eigenspace is $\operatorname{span}(\mathbf{1}_N)$. In the randomized gossip setting, the eigenvalues of the weighted graph Laplacian $\mathbf{L}(p)$ satisfy
\begin{equation*}
    0 = \lambda_1(p) \le \lambda_2(p) \le \cdots \le \lambda_N(p).
\end{equation*}
\begin{definition}[Spectral Gap]
For a given gossip matrix $\mathbf{W}$, the \emph{spectral gap} is defined as
\begin{equation*}
    \lambda_{\mathbf{W}} := \min_{\lambda \in \operatorname{Sp}\,(\mathbf{W}) \setminus \{1\}} 
    1 - |\lambda|.
\end{equation*}
\end{definition}

In the randomized gossip setting, since all eigenvalues of $\mathbf{L}(p)$ are non-negative, the spectral gap reduces to $\lambda_2(\mathbf{L}(p))/2$. This quantity is strictly positive if and only if $\mathcal{G}$ is connected. When connectivity fails, $\lambda_{\mathbf{W}} = 0$ and the convergence guarantee breaks down. The spectral gap therefore serves as the key parameter controlling the rate at which nodes reach consensus.

\subsection{Topology considerations}
More generally, if $\mathcal{G}$ decomposes into $K \geq 2$ disjoint connected components, then $\mathbf{L}(p)$ has exactly $K$ zero eigenvalues and the first non-zero eigenvalue is $\lambda_{K+1}(p)$. In this regime, a local consensus is reached within each component, but there is no guarantee that distinct components converge to the same value, so the global average $\bar{\mathbf{x}}$ is no longer recoverable. Whereas deterministic topologies such as rings or 2D grids are always connected by construction, random graph models may yield disconnected realizations, and care must be taken when using such constructions in experiments.

\paragraph{Small-world model (Watts--Strogatz) \cite{watts_collective_1998}.} The graph is built using the following random procedure. The graph $\mathcal G$ contains $N$ nodes and is initialized as a ring lattice where each node is linked to its $k$ nearest neighbors. Then, for each node $v\in\N $ and each of its $k/2$ rightmost edges, the edge is rewired with probability $p>0$ to a node $u\neq v$ chosen uniformly at random. For $k\gg \log(N)$, $\mathcal G$ is almost surely connected.

\paragraph{Random geometric graph \cite{gilbert_random_1961}.}
$N$ points are drawn uniformly at random in the unit square, and two points are connected if their Euclidean distance is below a threshold $r$.

\section{Preliminaries on randomized gossip protocol}
\label{sec:conv_gossip}
\paragraph{Markov chain considerations.} In this section, we first formalize the probabilistic constructions necessary to prove the main results. We consider the filtered probability space $(\Omega, \mathcal F, \mathbb P, (\mathcal F_t)_{t\ge 0})$, where the universe $\Omega$ contains the trajectories $\omega= (\omega_1, \omega_2, \dots)$ and $(\mathcal F_t)_{t\ge 0}$ is the filtration associated with the gossip stochastic process $(\mathbf X(t))_{t\ge 0}$. The randomness of a gossip process arises from the edge sampling process $e(\omega) = (e_t)_{t\ge 1}$. At each iteration $t>0$, the edge $e_t$ is drawn independently of the previous iterations, with $\mathbb P(e_t = e) = p_e > 0$. The gossip process is actually a Markov chain thanks to the construction of the edge sampling process. 
We denote $\mathbf x^{(i)}(t) := (x_{1,i}(t), \dots, x_{N,i}(t))^\top\in\mathbb R^N$ for $i\in I$ (component-wise estimation at iteration $t$ across all voters), $\mathbf x_v(t) := (x_{v,1}(t), \dots, x_{v,|I|}(t))^\top\in \mathbb R ^I $ (voter-wise embedding estimation at iteration $t$) for $v\in \N$ and finally $\mathbf X(t) := (\mathbf x_1(t),\dots, \mathbf x_N(t)) = (\mathbf x^{(1)}(t)^\top, \dots, \mathbf x^{(|I|)}(t)^\top)^\top$. For simplicity, we write $L^2$ for $L^2(\Omega, \mathcal F, \mathbb P)$.

 \paragraph{Main properties of randomized gossip processes.} For completeness, we recall standard properties of gossip matrices. We define the normalized vector $\mathbf z_N = \mathbf 1_N/\sqrt{N}$.
\begin{lemma}[\citet{boyd2006randomized}]
\label{lem:single-conv}
Let $I$ be a finite index set, $p = (p_e)_{e\in E}$. The vector $\mathbf{x}^{(i)}{(t)}$ evolves according to the gossip-based averaging procedure described in Section~\ref{sec:background_gossip}. Then, for any $i \in I$ and $t > 0$, we have
\[\mathbb{E}\left\|\mathbf{x}^{(i)}{(t)} - \bar x^{(i)}\mathbf{1}_N\right\|^2 \leq \left(1 - \frac{\lambda(p)}{2}\right)^{t} \left\| \mathbf x^{(i)}(0) - \bar x^{(i)}\mathbf{1}_N\right\|^2,\]
and thus $\mathbf{x}^{(i)}{(t)}$ converges towards $\bar x^{(i)}\mathbf{1}_N$ in $L^2$ when $t\to\infty$.
\end{lemma}

\begin{lemma}
\label{lem:properties}
Let $ t > 0 $.  If $e\in E $ is sampled with $e=(u,v)$, then the randomized gossip matrix at iteration $t$ is given by $ \mathbf W_t = \mathbf{I}_N - \frac 1 2 \mathbf L_e$, where $\mathbf L_e := \left( \mathbf{e}_u - \mathbf{e}_v \right) \left( \mathbf{e}_u - \mathbf{e}_v \right)^{\top}$.
The following properties hold:
\begin{enumerate}
    \item The matrices $\mathbf W_t$ are gossip matrices. Moreover, the matrices are idempotent, \textit{i.e.,} $\mathbf{W}_t^2 = \mathbf{W}_t$.
    \item Since  $\sum_{e \in E} p_e = 1$,  $\bar{\mathbf{W}} = \mathbb{E}[\mathbf{W}_t] = \mathbf{I}_N - \frac{1}{2}\sum_{e \in E} p_e\mathbf{L}_e.$
The matrix $\bar{\mathbf{W}}$ is also doubly stochastic and it follows that $\mathbf{1}_N$ is an eigenvector with eigenvalue 1. 
\item $\bar{\mathbf{W}}$ is positive semi-definite.
\item The matrix $\bar{\mathbf{W}} -{ \mathbf z_N \mathbf z_N^\top}$ satisfies, by construction, $(\bar{\mathbf{W}} - \mathbf z_N \mathbf z_N^\top) \mathbf{1}_N = 0$, and it can be shown that $\| \bar{\mathbf{W}} -\mathbf z_N \mathbf z_N^\top\|_{\operatorname{op}} \leq \lambda_{N-1}(\bar{\mathbf{W}})$ and $\|\cdot\|_{\operatorname{op}}$ denotes the operator norm of a matrix. Moreover, the eigenvalue $\lambda_{N-1}(\bar{\mathbf{W}})$ satisfies $0 \leq \lambda_{N-1}(\bar{\mathbf{W}}) < 1$ and $\lambda_{N-1}(\bar{\mathbf{W}}) = 1 - {\lambda(p)/2}$ where $\lambda(p)$ is the spectral gap (or second smallest eigenvalue) of the Laplacian of the weighted graph defined as $\mathbf L(p) =\sum_{e \in E} p_e\mathbf{L}_e$.
\end{enumerate}
\end{lemma}
\begin{proof}
For points \textit{1.}, \textit{2.}, and \textit{4.}, see \citet{boyd2006randomized}. For point \textit{3.}, observe that each individual gossip matrix $\mathbf{W}_e = \mathbf{I}_N - (\mathbf{e}_u - \mathbf{e}_v)(\mathbf{e}_u - \mathbf{e}_v)^\top/2$ is positive semi-definite. Thus, for any $\mathbf{y} \in \mathbb{R}^N$,
\[
\langle\mathbf y, \mathbf W_e \mathbf y\rangle = \|\mathbf{y}\|^2 - \frac{(y_u - y_v)^2}{2} \geq 0,\]
where the inequality follows from $(y_u - y_v)^2 \leq 2(y_u^2 + y_v^2) \leq 2\|\mathbf{y}\|^2$. Since $\bar{\mathbf{W}} = \sum_{e \in E} p_e \mathbf{W}_e$ is a convex combination of positive semi-definite matrices with $p_e > 0$ and $\sum_{e\in E} p_e = 1$, it is itself PSD.
\end{proof}

\section{Examples of the two-layer framework}
\label{sec:two_layers}
We now describe the general score-based framework underlying our algorithms. Let $I$ be an index set and consider mappings $\phi : \Sn \to \mathbb{R}^I$, $\psi : \mathbb{R}^I \to \mathbb{R}^n,$
where $\phi$ encodes a ranking and $\psi$ produces a score vector. The final ranking is $[\psi(\bar{\mathbf{x}})]_{\Sn}$, where $\bar{\mathbf{x}} = N^{-1}\sum_{v=1}^N\phi(\sigma_v)$ and the notation $[\,\cdot\,]_{\Sn}$ is defined in Section~\ref{sec:main_framework}. We present several such constructions in Table~\ref{tab:ranking_methods} (see \cite[Prop.~2]{xia_generalized_2008}).
\begin{table}[H]
\centering
\begin{center}
\begin{tabular}{l|ll|l}
\toprule
Method & Embedding $\phi(\sigma)$ & Aggregation $\psi(\mathbf{x})$ & Dim. $|I|$ \\
\midrule

Borda 
& $( \sigma(1), \dots,  \sigma(n))^\top$
& $-\mathrm{Id}$
& $n$ \\
\vspace{0.1em}
Copeland 
& $(\mathbb{I}\{\sigma(i) < \sigma(j)\})_{i < j}$ 
& $\left( \sum_{j \neq i} 
\big[ \mathbb{I}\{\mathbf x _{ij} > \tfrac{1}{2}\} 
- \mathbb{I}\{\mathbf x_{ij} < \tfrac{1}{2}\} \big] 
\right)_{i=1}^n$ 
& $\binom n 2$ \\

Bucklin \cite{erdelyi_control_2015} & $(\mathbb I \{\sigma(i) = k\})_{i,k\in\n}$ & $\left(-\min\left\{k\in\n\mid \sum^k_{\ell=1} \mathbf x_{i\ell} > 1/2\right\} \right)_{i=1}^n$ & $n^2$\\
Plurality \cite{Brandt_Conitzer_Endriss_Lang_Procaccia_2016}
& $(\mathbb I\{\sigma(i)=1\})_{i\in\n}$
& $\mathrm{Id}$
& $n$ \\

Maximin \cite{simpson1969}
& $(\mathbb{I}\{\sigma(i) < \sigma(j)\})_{i < j}$
& $\left( \min_{j \neq i} \mathbf x_{ij} \right)_{i=1}^n$
& $\binom n 2$ \\
\bottomrule
\end{tabular}
\end{center}
\caption{Common ranking methods expressed in embedding-aggregation form. 
The map $\phi$ encodes a permutation $\sigma\in\Sn$ into the embedding space, 
$\psi$ aggregates information into scores, 
and $|I|$ denotes the dimensionality of the embedding.}
\label{tab:ranking_methods}
\end{table}

\begin{remark}[Communication cost]
\label{rem:comm_cost}
At each gossip round, each node exchanges a vector of dimension $|I|$. For Borda and plurality, $|I| = n$, for Copeland and Maximin, $|I| = \binom n 2$, for Bucklin, $|I| = n^2$ (see Table~\ref{tab:ranking_methods}). In contrast, a centralized approach requires collecting all $N$ rankings at a single node, incurring $O(Nn)$ total communication. The gossip approach replaces this with $O(T |I|)$ per-node cost, where $T$ is the mixing time.
\end{remark}

Proposition~\ref{prop:lipschitz} establishes a useful regularity property of the stability radius, regardless of the score-based voting rule.
\begin{proposition}\label{prop:lipschitz}
For a given function $\psi$, the function $\mathbf x \mapsto r_\psi(\mathbf x)$ is $1$-Lipschitz on $\mathbb R^I$.
\end{proposition}
\begin{proof}
We show that $\mathbf x \mapsto r_\psi(\mathbf x)$ is $1$-Lipschitz, which implies continuity. Let $\mathbf x, \mathbf x' \in\mathbb R ^I$. We can write the stability radius as $r_\psi(\mathbf x) =  \inf_{\mathbf z \in T(\mathbf x)}\|\mathbf x - \mathbf z\|,$ where $T(\mathbf x)  := \{\mathbf y \in\mathbb R^I\mid [\psi(\mathbf y)]_{\Sn} \ne[\psi(\mathbf x)]_{\Sn} \}$.
If $[\psi(\mathbf x)]_{\Sn} = [\psi(\mathbf x')]_{\Sn}$, then $T(\mathbf x) = T(\mathbf x')$. Since the distance to a fixed set is $1$-Lipschitz:
\[
 |r_\psi(\mathbf x) - r_\psi(\mathbf x')| \le \|\mathbf x - \mathbf x'\|.
\]
Otherwise, we get $\mathbf x'\in T(\mathbf x)$ and $\mathbf x\in T(\mathbf x')$. Following the definition of $T(\mathbf x)$,
$r_\psi(\mathbf x) \le \|\mathbf x - \mathbf x'\|$ and  $r_\psi(\mathbf x') \le \|\mathbf x - \mathbf x'\|.$
Therefore, $\max(r_\psi(\mathbf x), r_\psi(\mathbf x'))\le \|\mathbf x - \mathbf x'\|.$
Since $\max(r_\psi(\mathbf x), r_\psi(\mathbf x'))\ge |r_\psi(\mathbf x)- r_\psi(\mathbf x')|$, we conclude that $r_\psi$ is $1$-Lipschitz in both cases.
\end{proof}
\begin{remark}
Although $r_\psi(\bar{\mathbf{x}})$ depends on the unknown global average $\bar {\mathbf x}$, Proposition~\ref{prop:lipschitz} allows each node to estimate it consistently from its local gossip iterate.
More precisely, the Lipschitz property gives $\mathbb E |r_\psi(\mathbf{x}_v(t)) - r_\psi(\bar{\mathbf{x}})|^2 \le \mathbb E \|\mathbf{x}_v(t) - \bar{\mathbf{x}}\|^2,$ which decays at the same rate as the gossip process itself. This enables  practical estimation of the stability radius appearing in the mixing time bounds of Remarks~\ref{rem:sync-mixing-time} and~\ref{rem:mixing-time}, each node can compute $r_\psi(\mathbf{x}_v(t))$ locally to approximate the global threshold $r_\psi(\bar{\mathbf{x}})$ and implement adaptive stopping criteria without coordination.
\end{remark}

\section{Proofs of main results}
\label{sec:proofs}

\subsection{Convergence of shift-register gossip via Chebyshev polynomials}
\label{sec:chebyshev}
Let $\mathbf W$ be a gossip matrix (see Appendix~\ref{sec:graph_topology}), with eigenvalues $-1\le\lambda_1\le\cdots\le\lambda_{N-1}\le\lambda_N=1$, where $\lambda_k:=\lambda_k(\mathbf W)$, and an orthonormal basis of eigenvectors $(\mathbf z_k)_{k=1}^N$ with $\mathbf z_N=\mathbf 1_N/\sqrt N$.

Given initial values $\mathbf a\in\mathbb R^N$ with average $\bar a=(1/N)\sum_{v=1}^N a_v$, we consider the shift-register (second-order) gossip iteration with parameter $\alpha$:
\begin{equation}\label{eq:shift_register}
\mathbf x(0)=\mathbf a,\qquad \mathbf x(1)=\mathbf W\mathbf a,\qquad \mathbf x(t+1)=\alpha\mathbf W\mathbf x(t)+(1-\alpha)\mathbf x(t-1),\quad t\ge 1 .
\end{equation}
The first step is a plain gossip step. We use the optimal parameter of \cite{liu_analysis_2013}, written in terms of $\lambda_{\mathbf W}$:
\begin{equation}\label{eq:alpha_star}
\alpha_\star=2\,\frac{1-\sqrt{\lambda_{\mathbf W}(1-\lambda_{\mathbf W}/4)}}{(1-\lambda_{\mathbf W}/2)^2} = \frac{2}{1+\sqrt{\lambda_{\mathbf W}(1-\lambda_{\mathbf W}/4)}}.
\end{equation}
It is worth noting that for $\alpha' := \sqrt{\lambda_{\mathbf W}(1-\lambda_{\mathbf W}/4)}$ and $1-\alpha'^2 =(1- \lambda_\mathbf W/2)^2$. Moreover, $\alpha_\star = (1- \alpha'^2 )/(1+\alpha')^2$ and thus \begin{equation}\label{eq:alpha_relation}
    \frac{\alpha_\star -1} {\alpha_\star^2} = \frac {1 - \alpha'^2}{4} = \frac{(1-\lambda_\mathbf W/2)^2}{4}.
\end{equation}
\begin{proposition}[Convergence rate]\label{prop:chebyshev}
Assume $\lambda_{\mathbf W}\in(0,1]$, and let $\mathbf y(t)$ be the iterates of \eqref{eq:shift_register} with $\alpha=\alpha_\star$ given by \eqref{eq:alpha_star}. Then $\alpha_\star\in(1,2)$ and, for all $t\ge0$,
\[
\|\mathbf x(t)-\bar a\mathbf 1_N\|^2\le2(\alpha_\star-1)^t\|\mathbf a-\bar a\mathbf 1_N\|^2\le2e^{-\sqrt{\lambda_{\mathbf W}}t}\|\mathbf a-\bar a\mathbf 1_N\|^2 .
\]
\end{proposition}


The proof of Proposition~\ref{prop:chebyshev}, presented below, is inspired by elements of \cite{berthier_accelerated_2019,liu_analysis_2013}.

\begin{lemma}[Polynomial representation]\label{lem:cheb_poly}
Define polynomials by $P_0=1$, $P_1=X$ and $P_{t+1}=\alpha_\star XP_t+(1-\alpha_\star)P_{t-1}$ for $t\ge1$. Then $\mathbf y(t)=P_t(\mathbf W)\mathbf a$, $P_t(1)=1$, and
\[
\|\mathbf x(t)-\bar a\mathbf 1_N\|^2=\sum_{k=1}^{N-1} |P_t(\lambda_k)|^2\langle \mathbf a,\mathbf z_k\rangle^2 .
\]
\end{lemma}

\begin{proof}
The identity $\mathbf x(t)=P_t(\mathbf W)\mathbf a$ holds for $t=0,1$, and it propagates through \eqref{eq:shift_register} by induction. We have $P_0(1)=P_1(1)=1$, and if $P_t(1)=P_{t-1}(1)=1$ then $P_{t+1}(1)=\alpha_\star+1-\alpha_\star=1$. Writing $\mathbf a=\bar a\mathbf 1_N+\sum_{k\le N-1}\langle \mathbf a,\mathbf z_k\rangle \mathbf z_k$ and using $P_t(\mathbf W)\mathbf 1_N=\mathbf 1_N$ and $P_t(\mathbf W)\mathbf z_k=P_t(\lambda_k)\mathbf z_k$, we get $\mathbf x(t)-\bar a\mathbf 1_N=\sum_{k\le N-1}P_t(\lambda_k)\langle \mathbf a,\mathbf z_k\rangle \mathbf z_k$. The formula follows from Pythagora's theorem.
\end{proof}

\begin{lemma}\label{lem:cheb_unif}
Let $\lambda\in \operatorname{Sp}(\mathbf W)\setminus\{1\}$ and $t\ge 0$, then $|P_t(\lambda)|^2 \le 2 (\alpha_\star-1)^t.$
\end{lemma}
\begin{proof}
    Let $\lambda\in \operatorname{Sp}(\mathbf W) \setminus \{1\}$ and $t\ge 1$, the Chebyshev polynomials $P_t(\lambda)$ verify the second-order linear relation \cite{mason2002chebyshev}:
    \begin{equation}\label{eq:chebyshev_2order}
    P_{t+1}(\lambda) = \alpha_\star\lambda P_t(\lambda) + (1-\alpha_\star) P_{t-1}(\lambda).
    \end{equation}
    Since the discriminant $\alpha_\star^2 \lambda^2 - 4(\alpha_\star-1)$ is negative for all eigenvalues, we know the closed form of the relation in Equation~\eqref{eq:chebyshev_2order} is given for all $t\ge 0$ by $P_t(\lambda) = A_\lambda(r^+_\lambda)^t + B_\lambda (r^-_\lambda)^t $, where $A_\lambda, B_\lambda$ are two constants and:
    \[
        r^\pm_\lambda = \frac {\alpha_\star\lambda}{      
        2} \pm i \sqrt {\alpha_\star - 1 - \alpha_\star^2 \lambda^2/4}.
    \]
    We note that $|r^\pm_\lambda|^2 = \alpha_\star - 1$. We solve the linear system whose constants are given by $P_0(\lambda) =1$ and $P_1(\lambda) = \lambda$ to find $A_\lambda = (r_\lambda^- - \lambda) / (r_\lambda^- - r_\lambda^+ )$ and $B_\lambda = 1 - A_\lambda$.
    Given the previous expressions, it is obvious that $A_\lambda = \overline{B_\lambda}$ and $r_\lambda^+ = \overline{r_\lambda^-}$. Thus, $|A_\lambda| = |B_\lambda|$ and $|r^+_\lambda| = |r^-_\lambda|$, which leads to:
    \begin{multline}
        |P_t(\lambda)|^2 \le  4|A_\lambda|^2|r^+_\lambda|^{2t}
        = \frac {4(\alpha_\star -1)(1-\lambda^2)}{4(\alpha_\star - 1) - \alpha_\star^2 \lambda^2}(\alpha_\star - 1)^{t}
        = \frac{(1-\alpha'^2)(1-\lambda^2) }{1-\alpha'^2 - \lambda^2} (\alpha_\star - 1)^{t},\\
        = \frac{(1-\alpha'^2)(1-\lambda^2 -\alpha'^2 + \alpha'^2) }{1-\alpha'^2 - \lambda^2} (\alpha_\star - 1)^{t}
        = (1-\alpha'^2)\left(1  + \frac{\alpha'^2 }{1-\alpha'^2 - \lambda^2} \right)(\alpha_\star - 1)^{t}.\label{eq:bound_poly_alpha}
    \end{multline}

    Since $\lambda^2\le (1 - \lambda_{\mathbf W})^2\le (1 - \alpha'^2)^2 $, we can conclude using Equation~\eqref{eq:bound_poly_alpha}:
    \[
    |P_t(\lambda)|^2\le (2-\alpha'^2)(\alpha_\star - 1)^t\le 2 (\alpha_\star -1)^t.
    \]
\end{proof}

\begin{proof}[Proof of Proposition~\ref{prop:chebyshev}]
Lemmas~\ref{lem:cheb_poly} and \ref{lem:cheb_unif} give
\begin{equation*}
    \|\mathbf x(t)-\bar a\mathbf 1_N\|^2=\sum_{k=1}^{N-1}|P_t(\lambda_k)|^2\langle \mathbf a,\mathbf z_k\rangle^2\le2(\alpha_\star -1)^{t}\sum_{k=1}^{N-1}\langle \mathbf a,\mathbf z_k\rangle^2=2(\alpha_\star-1)^t\|\mathbf a-\bar a\mathbf 1_N\|^2 .
\end{equation*}

For the remaining inequalities, note that $\alpha_\star-1=\frac{1-\alpha'}{1+\alpha'}=\exp(-2\operatorname{artanh}\alpha')$ and, since $\alpha'\ge0$, $\operatorname{artanh}\alpha'\ge \alpha'$, so $\alpha_\star-1\le e^{-2\alpha'}$. Since $\lambda_{\mathbf W}\le1$, we have $1-\lambda_{\mathbf W}/4\ge3/4$, hence $\alpha'\ge\frac{\sqrt3}{2}\sqrt{\lambda_{\mathbf W}}$ and $2\alpha'\ge\sqrt3\sqrt{\lambda_{\mathbf W}}\ge\sqrt{\lambda_{\mathbf W}}$. Therefore $(\alpha_\star-1)^t\le e^{-\sqrt3\sqrt{\lambda_{\mathbf W}}\,t}\le e^{-\sqrt{\lambda_{\mathbf W}}\,t}$.
\end{proof}

\subsection{Proof of Theorem \ref{thm:sync-general-ranking-convergence}}
We first establish the following auxiliary result, which is useful for the convergence guarantees for both the synchronous and randomized gossip protocols. Throughout, we write $\mathbb{E}_{P}f = \sum_{\sigma\in\Sn} P(\sigma)f(\sigma)$ for any measurable map $f$ from $\Sn$ to a Euclidean space.
\begin{lemma}
    \label{lem:kendall-err}
    For a graph $\mathcal G =(\N, E)$. Let $(\mathbf{x}_v(t))_{t \ge 0,\, v \in \N}$ be the distributed stochastic process generated by a $(\phi, \psi)$-voting rule. Then, for all $t \ge 0$, the average expected Kendall-$\tau$ error satisfies
        \[\mathbb{E}\left[\frac{1}{N}\sum_{v=1}^N d_\tau\left(\sigma_\star,\, \widehat{\sigma}_v(t)\right)\right]\le\frac{\binom{n}{2}}{Nr_\psi(\bar{\mathbf{x}})^2}\sum_{i \in I} \mathbb{E}\,\left\|\mathbf x^{(i)}(t) - \bar{x}^{(i)}\,\mathbf{1}_N\right\|^2.\]
\end{lemma}
\begin{proof}
    Let $t\ge 0$ and a $(\phi, \psi)$-voting rule. For node $v\in\N$, we need to consider the set of per-node events $\Omega(\mathbf x_v(t))$ such that the running estimate $\mathbf x_v(t)$ is inside the stability radius of estimate $\bar{\mathbf x}=\mathbb E _{\widehat P_N}\phi$ which is assumed to be fixed and where
\[
\Omega(\mathbf x) :=\{\|\mathbf x - \bar{\mathbf x }\|\le r_\psi(\bar{\mathbf x})\} = \{\|\mathbf x - \bar{\mathbf x }\|^2\le r_\psi^2(\bar{\mathbf x})\}.
\]
Then, for a node $v\in\N$, the expected Kendall-$\tau$ error satisfies the conditional expectation relation:
\begin{multline*}
\mathbb E\, d_\tau\left(\sigma_\star, \widehat{\sigma}_v(t)\right) =  \mathbb{E}\left[d_\tau\left(\sigma_\star, \widehat{\sigma}_v(t)\right)\,\middle|\, \Omega(\mathbf x_v(t))\right]\,\mathbb P (\Omega(\mathbf x_v(t)))\\+\mathbb{E}\left[d_\tau\left(\sigma_\star, \widehat{\sigma}_v(t)\right)\,\middle|\,\Omega^c(\mathbf x_v(t))\right]\mathbb P(\Omega^c(\mathbf x_v(t))),
\end{multline*}
where $\Omega^c(\mathbf x_v(t))$ is the complementary of $\Omega(\mathbf x_v(t))$ in $\Omega$. In $\Omega(\mathbf x_v(t))$, the $L^2$-error of gossip variables is lower than the stability radius for all nodes, then we can conclude that $\mathbb{E}\left[d_\tau\left(\sigma_\star, \widehat{\sigma}_v(t)\right)\,\middle|\, \Omega(\mathbf x_v(t))\right]\, = 0.$ Furthermore, because the Kendall-$\tau$ distance is upper bounded by $\binom n 2$, then we can upper bound the error as follows:
\begin{equation}
\label{eq:kendall_err_node}
 \mathbb E \,d_\tau\left(\sigma_\star, \widehat{\sigma}_v(t)\right) \le \binom n 2 \mathbb{P}(\Omega^c(\mathbf x_v(t))).
\end{equation}
Applying Markov's inequality on event $\Omega^c(\mathbf x_v(t))$ yields:
\begin{equation}
\label{eq:markov_node}\mathbb P (\Omega^c(\mathbf x_v(t))) \le  \frac{\mathbb E \left\|\mathbf x_{v}(t) - \bar{\mathbf x}\right\|^2}{r_\psi(\bar{\mathbf x})^2} = \sum_{i\in I} \frac{\mathbb E \left(x_{v,i}(t) - \bar{x}^{(i)}\right)^2}{r_\psi(\bar{\mathbf x})^2}.
\end{equation}
Substituting \eqref{eq:markov_node} into \eqref{eq:kendall_err_node} and averaging over all nodes $v \in \N$:
\begin{align*}
\mathbb{E}\left[\frac{1}{N}\sum_{v=1}^N d_\tau\left(\sigma_\star, \widehat{\sigma}_v(t)\right)\right]
&\le \frac{\binom n 2}{N} \sum_{v=1}^N \frac{1}{r_\psi(\bar{\mathbf{x}})^2} \sum_{i \in I} \mathbb{E}\left(x_{v,i}(t) - \bar{x}^{(i)}\right)^2, \\
&= \frac{\binom{n}{2}}{r_\psi(\bar{\mathbf{x}})^2} \sum_{i \in I} \frac{1}{N}\sum_{v=1}^N \mathbb{E}\left(x_{v,i}(t) - \bar{x}^{(i)}\right)^2, \\
&= \frac{\binom{n}{2}}{Nr_\psi(\bar{\mathbf{x}})^2} \sum_{i \in I} \,\mathbb{E}\left\|\mathbf{x}^{(i)}(t) - \bar{x}^{(i)}\mathbf{1}_N\right\|^2,
\end{align*}
which completes the proof.
\end{proof}
\begin{proof}[Proof of Theorem~\ref{thm:sync-general-ranking-convergence}]
Let $t\ge 0$ and $i\in I$. In Algorithm~\ref{alg:sync-decentralized-cons}, the auxiliary variable $\mathbf x'_v$ stores the previous estimate of node $v$, the first round is a plain gossip step, and the subsequent rounds use the memory coefficient $\alpha_\star$. Hence the vector $\mathbf x^{(i)}(t)\in\mathbb R^N$ of the $i$-th coordinates across all nodes follows the shift-register iteration~\eqref{eq:shift_register} with $\mathbf a=\mathbf x^{(i)}(0)$ and $\bar a=\bar x^{(i)}$:
\[
\mathbf x^{(i)}(1)=\mathbf W\mathbf x^{(i)}(0),\qquad \mathbf x^{(i)}(t+1)=\alpha_\star\mathbf W\mathbf x^{(i)}(t)+(1-\alpha_\star)\mathbf x^{(i)}(t-1),\quad t\ge1 .
\]
Since $\lambda_{\mathbf W}\in(0,1]$, Proposition~\ref{prop:chebyshev} gives
\begin{equation}\label{eq:chebyshev_bound}
\|\mathbf x^{(i)}(t)-\bar x^{(i)}\mathbf 1_N\|^2\le 2\,e^{-\sqrt{\lambda_{\mathbf W}}\,t}\,\|\mathbf x^{(i)}(0)-\bar x^{(i)}\mathbf 1_N\|^2 .
\end{equation}
The synchronous process is deterministic, so the expectation in Lemma~\ref{lem:kendall-err} is trivial. Combining Lemma~\ref{lem:kendall-err} with \eqref{eq:chebyshev_bound} summed over $i\in I$, we obtain
\[
\frac{1}{N}\sum_{v=1}^N d_\tau\left(\sigma_\star, \widehat{\sigma}_v(t)\right)\le\frac{\binom n2}{Nr_\psi(\bar{\mathbf{x}})^2}\sum_{i\in I}\|\mathbf x^{(i)}(t)-\bar x^{(i)}\mathbf 1_N\|^2\le\frac{2\binom n2}{Nr_\psi(\bar{\mathbf{x}})^2}\,e^{-\sqrt{\lambda_{\mathbf W}}\,t}\,\gamma. \qedhere
\]
\end{proof}
\subsection{Proof of Remark~\ref{rem:sync-mixing-time}}
\begin{proof}
Since $d_\tau$ is integer-based, we note the equivalence between $N^{-1}\sum_{v=1}^N d_\tau\left(\sigma_\star, \widehat{\sigma}_v(t)\right) = 0$ and $N^{-1}\sum_{v=1}^N d_\tau\left(\sigma_\star, \widehat{\sigma}_v(t)\right) < \frac 1 N.$
Applying Theorem~\ref{thm:sync-general-ranking-convergence} yields a sufficient condition for the previous condition on the average of the Kendall-$\tau$ distances to happen:
\[\frac{2\binom n 2 }{Nr_\psi(\bar{\mathbf x})^2} e^{-\sqrt {\lambda_{\mathbf W}}  t}\gamma< \frac 1 N .\]
This condition holds for every integer $t>t^\star$, where $t^\star :=\lambda_{\mathbf W}^{-1/2}\log \left(\frac {2\binom n 2 \gamma}{r^2_\psi(\bar{\mathbf x})}\right)$. If $t^\star<0$, it already holds at $t=0$. Otherwise, the smallest integer $t>t^\star$ is $\lfloor t^\star\rfloor+1$. Since $T^{\mathrm{Sync}}(0)$ is the smallest time such that the average of the Kendall-$\tau$ distance is null, we obtain the desired bound $T^\mathrm{Sync}(0) \le \max\{0,\lfloor t^\star\rfloor+1\}$.
\end{proof}

\subsection{Proof of Theorem~\ref{thm:general-ranking-convergence}}
\begin{proof}
Let $t\ge 0$. Applying Lemma~\ref{lem:kendall-err} directly yields:
\begin{equation}\label{eq:bound_dtau}
    \mathbb{E}\left[\frac{1}{N}\sum_{v=1}^N d_\tau\left(\sigma_\star, \widehat{\sigma}_v(t)\right)\right]\le\frac{\binom{n}{2}}{Nr_\psi(\bar{\mathbf{x}})^2}\sum_{i \in I} \mathbb{E}\,\left\|\mathbf x^{(i)}(t) - \bar{x}^{(i)}\,\mathbf{1}_N\right\|^2.
\end{equation}
We then apply Lemma~\ref{lem:single-conv} to bound each term in the right-hand side sum:
\begin{equation}\label{eq:gossip_bound_embed}
    \sum_{i \in I} \mathbb{E}\left\|\mathbf x^{(i)}(t) - \bar{x}^{(i)}\,\mathbf{1}_N\right\|^2\le\gamma\left(1 - \frac{\lambda(p)}{2}\right)^t. 
\end{equation}
Combining Equation~\eqref{eq:bound_dtau} and Equation~\eqref{eq:gossip_bound_embed} gives:
\[\mathbb{E}\left[\frac{1}{N}\sum_{v=1}^N d_\tau\left(\sigma_\star, \widehat{\sigma}_v(t)\right)\right]\le\frac{\binom n 2\gamma}{Nr_\psi(\bar{\mathbf{x}})^2}\left(1 - \frac{\lambda(p)}{2}\right)^t.\]
The conclusion follows from the convex inequality $1 - \lambda(p)/2 \le \exp(-\lambda(p)/2)$.
\end{proof}

\subsection{Proof of Remark~\ref{rem:mixing-time}}
\begin{proof}
For the upper bound, we use similar arguments as the ones used for Remark~\ref{rem:sync-mixing-time}. By Theorem~\ref{thm:general-ranking-convergence} and since $N\ge1$, the expected average Kendall-$\tau$ error is at most $\binom n2\gamma\, e^{-\lambda(p)t/2}/r_\psi^2(\bar{\mathbf x})$, which is at most $\delta$ for every integer $t> t^\star_{\mathrm{rand}}$, where
\[
t^\star_{\mathrm{rand}}:=\frac {2}{\lambda(p)}\log\left(\binom n 2 \frac {\gamma}{\delta r_\psi^2(\bar {\mathbf x})}\right).
\]
If $t^\star_{\mathrm{rand}}<0$, this holds at $t=0$. Otherwise, the smallest such integer is $\lfloor t^\star_{\mathrm{rand}}\rfloor+1$. Hence $T^\mathrm{rand}(\delta)\le\max\{0,\lfloor t^\star_{\mathrm{rand}}\rfloor+1\}$.

For the lower bound, we note that at each iteration $t+1\ge 0$ of the randomized gossip protocol only two nodes share their embedding and we have:
\[\frac 1 N \sum_{v=1}^N d_\tau(\widehat \sigma_v(t+1), \sigma_\star) - \frac 1 N \sum_{v=1}^Nd_\tau(\widehat \sigma_v(t), \sigma_\star) \ge -\frac {n(n-1)} N.\]
Telescoping the sums and taking the expectation yields:
\[\frac 1 N \sum_{v=1}^N \mathbb E\, d_\tau(\widehat \sigma_v(t+1), \sigma_\star) - \frac 1 N \sum_{v=1}^Nd_\tau(\widehat \sigma_v(0), \sigma_\star) \ge -\frac {n(n-1)} Nt.\]
At iteration $T^{\mathrm{rand}}(\delta)$, we can rewrite the previous inequality as $\delta \ge N^{-1} \sum_{v=1}^Nd_\tau(\widehat \sigma_v(0), \sigma_\star) - n(n-1)T^{\mathrm{rand}}(\delta)/N$, which yields 
\[T^{\mathrm{rand}}(\delta) \ge  \max\left(\frac {\sum_{v=1}^Nd_\tau(\widehat \sigma_v(0), \sigma_\star) - N\delta}{n(n-1)}, 0\right).\]
\end{proof}

\subsection{Proof of Proposition~\ref{prop:bound_breakdown}}
\label{sec:proof_bound_breakdown}
\begin{proof}
We consider a score function $S$ that admits a $(\phi, \psi)$-decomposition. Let $1\le\delta\le \binom n 2$. We now restate the empirical breakdown function and define the associated set $\mathcal E$:
\begin{equation*}
    \varepsilon_{\widehat P_N, S}^{\star}(\delta)=\inf \mathcal E:=\inf \left\{\left.\frac{b}{N} \right\rvert\, b \in\N, \sup _{\widehat P_N^\varepsilon \in \tilde B(\widehat P_N, b/N)} d_\tau\left(S\left(\widehat P_N\right), S\left(\widehat P_N^\varepsilon\right)\right) \geq \delta\right\}.
\end{equation*}
Let $\varepsilon\in \mathcal E$,
\begin{align}
    \varepsilon\in \mathcal E &\iff \exists\, \widehat P_N^\varepsilon \in \Delta_{\Sn} \text{ such that } d_\mathrm{TV}(\widehat P_N, \widehat P_N^\varepsilon) \le \varepsilon, d_\tau(\psi(\mathbb E _{\widehat P_N^\varepsilon}\phi), \psi(\mathbb E _{\widehat P_N}\phi))\ge \delta,\nonumber\\
    &\implies \exists\, \widehat P_N^\varepsilon \in \Delta_{\Sn} \text{ such that } d_\mathrm{TV}(\widehat P_N, \widehat P_N^\varepsilon) \le \varepsilon,\|\mathbb E _{\widehat P_N^\varepsilon}\phi - \mathbb E _{\widehat P_N}\phi\|\ge r_\psi\left(\mathbb E _{\widehat P_N}\phi\right).\label{eq:radius_bound}
\end{align}
Furthermore, we can bound the norm of the difference of expectations:
\begin{align}
\left\|\mathbb E_{\widehat P_N^\varepsilon}\phi - \mathbb E_{\widehat P_N}\phi\right\|
&\le \left\|\sum_{\sigma\in\Sn} \left[\widehat P_N^\varepsilon(\sigma) - \widehat P_N(\sigma)\right]\phi(\sigma)\right \|\le \sum_{\sigma\in\Sn}|\widehat P_N^\varepsilon(\sigma) - \widehat P_N(\sigma)| \|\phi(\sigma)\|\nonumber\\&\le 2d_\mathrm{TV}(\widehat P_N , \widehat P_N^{\varepsilon}) \max_{\sigma\in\Sn}\|\phi(\sigma)\| \le 2\varepsilon \trinorm{\phi}.\label{eq:holder}
\end{align}
Using both Equation~\eqref{eq:radius_bound} and Equation~\eqref{eq:holder}:
\[r_\psi\left(\mathbb E _{\widehat P_N}\phi\right) \le \left\|\mathbb E_{\widehat P_N^\varepsilon}\phi - \mathbb E_{\widehat P_N}\phi\right\| \le 2\varepsilon  \trinorm{\phi},\]
which implies $\varepsilon \ge \frac {r_\psi\left(\mathbb E_{\widehat P_N} \phi\right)}{2 \trinorm{\phi}}$. Taking the infimum on $E$ allows us to conclude on the lower bound.
\end{proof}

\subsection{Proof of Theorem~\ref{thm:timed_break}}
\label{sec:proof_timed_break}

Let $\delta > 0$ and a given set $B\subseteq \N$ of size $b= |B|$ and $H = \N \setminus B$. The gossip dynamics are governed by the cocycle $\mathbf W_{0:t} := \prod_{s=1}^{t} \mathbf W_s := \mathbf W_t \cdots \mathbf W_1$, where each $\mathbf W_s$ is a gossip matrix, drawn independently with $\mathbb{E}\mathbf W_s = \bar{\mathbf W}$ for all $s \ge 1$. The breakdown event is defined in terms of the consensus output $\bar{\mathbf{x}}_h(t) := \frac 1 {N-b}\sum_{v\in H}\mathbf{x}_v(t)$ compared to the initial honest consensus $\bar{\mathbf{x}}_h = \frac 1 {N-b}\sum_{v\in H} \mathbf x_v(0)$ is governed by the stability radius $r_\psi(\bar{\mathbf{x}}_h)$. Let $\mathbf{y}_u = \mathbf{x}_u(0)-\bar{\mathbf{x}}_h$ be the initial error at node $u\in\N$.
The strategy of the proof is to make appear the fraction $\frac b N$ in an upper bound of the expectation term of Equation~\eqref{eq:cond_x-x} using the decomposition:
\begin{equation}\label{eq:decomp_error}
\mathbb E \left\| \bar{\mathbf x}_h(t)  - \bar{\mathbf{x}}\right\| \le \mathbb E \left\| \bar{\mathbf x}_h(t)  - \bar{\mathbf{x}}_h\right\| + \left\| \bar{\mathbf x}_h  - \bar{\mathbf{x}}\right\|.
\end{equation}

Furthermore, by writing $\bar {\mathbf x}$ as a barycenter of $\bar{\mathbf x}_h$ and $\bar{\mathbf x}_b$, one can find \begin{equation}\label{eq:bound_means}
\|\bar{\mathbf x}_h - \bar{\mathbf x}\| = \left\| \bar{\mathbf x}_h - \frac b N \bar {\mathbf x}_b - \frac {N-b}{N} \bar{\mathbf x}_h \right\|=  \frac b N \|\bar{\mathbf x}_h - \bar{\mathbf x}_b\| \le 2 \frac b N \trinorm \phi.
\end{equation}

Thus, the only difficulty is to bound the term $\mathbb E \left\| \bar{\mathbf x}(t)  - \bar{\mathbf{x}}_h\right\|$. In order to achieve so, we consider a second decomposition.
Since $\mathbf W_{0:t}$ is doubly stochastic and initial values are fixed, we write:
\begin{align}\label{eq:decomp}
\bar{\mathbf x}_h(t) - \bar{\mathbf{x}}_h 
&= \frac 1 {N-b} \sum_{v\in H} \mathbf x_v(t) - \bar{\mathbf x}_h = \frac 1 {N-b} \sum_{(u,v)\in \N \times H}[ \mathbf W_{0:t}]_{vu}(\mathbf x_v(0) - \bar{\mathbf x}_h),
\nonumber\\&= \underbrace{\frac 1 {N-b} \sum_{(u,v)\in H^2}[\mathbf W_{0:t}]_{vu}\mathbf{y}_u}_{=:\mathbf y_H}+ \underbrace{\frac 1 {N-b} \sum_{(w, v)\in B \times H}[\mathbf W_{0:t}]_{vw}\mathbf{y}_w.}_{=:\mathbf y_B}
\end{align}
The strategy is now to find a way to bound $\mathbf y_H$ which represent the uncontaminated dynamic of the process, while $\mathbf y_B$ denotes the effect of the added contamination over the rest of the nodes. We note that the non-nullity of $\mathbf y_H$ arises from the assumption that $\mathcal G_H$ is connected. Lemma~\ref{lem:honest_bound} bounds the first term while Lemma~\ref{lem:Qt_bound} bounds the effect of the contamination on the error of the process.

\begin{lemma}\label{lem:honest_bound}
Given the definition of $\mathbf y_H$ in Equation~\eqref{eq:decomp}, then:
\[\mathbb{E}\|\mathbf y_H\| \le 2\trinorm{\phi}\sqrt{\frac{b}{N}}\xi_t.\]
\end{lemma}

\begin{proof}
Let $\mathbf{h} := \frac{1}{N-b}\mathbf{1}_H$ be the normalized indicator of honest nodes.  Swapping the order of summation gives:
\begin{equation}\label{eq:sum_yH}
\mathbf y_H = \sum_{u\in H} \left(\frac{1}{N-b}\sum_{v\in H}[\mathbf{W}_{0:t}]_{vu}\right)\mathbf{y}_u = \sum_{u\in H} [\mathbf{W}_{0:t}\mathbf{h}]_u\mathbf{y}_u.
\end{equation}
The right hand-side term is the weighted sum of the vector family $(\mathbf y_u)_{u\in H},$ whose weights are $a_u := [\mathbf{W}_{0:t}\mathbf{h}]_u$ for $u\in H$. Given a scalar family $(a_u)_{u\in H}$ and a vector family $(\mathbf y_u)_{u\in H}$, we recall that using Cauchy-Schwarz inequality twice yields the following inequality:
\begin{multline}\label{eq:cs_eq}
    \left\|\sum_{u\in H}a_u \mathbf y_u\right\|^2 = \sum_{(u,v)\in H^2} a_u a_v  \langle\mathbf y_u,\mathbf y_v\rangle \le  \sum_{(u,v)\in H^2} a_u a_v  \|\mathbf y_u\|\|\mathbf y_v\| \\= \left(\sum_{u\in H} a_u\|\mathbf y_u\| \right) \left(\sum_{u\in H} a_u\|\mathbf y_u\| \right) \le \left(\sum_{u\in H}a_u^2\right)\left(\sum_{u\in H}\|\mathbf y_u\|^2\right).
\end{multline}

Since $\sum_{u\in H}\mathbf{y}_u=0$ we may again remove $1/N$ to Equation~\eqref{eq:sum_yH}, and Equation~\eqref{eq:cs_eq} gives:
\[
\|\mathbf y_H\|^2 \le \underbrace{\sum_{u\in H }\left((\mathbf{W}_{0:t}^\top\mathbf{h})_u - \tfrac{1}{N}\right)^2}_{=\left\|\left[\mathbf{W}_{0:t}^\top-\mathbf z_N \mathbf z_N^\top\right]\mathbf{h}\right\|^2}\left(\sum_{u\in H}\|\mathbf{y}_u\|^2\right).
\]
We now can upper bound the norm on the subspace $ \operatorname{span}(\mathbf 1_N)^\perp$. Following Lemma~\ref{lem:properties}, $\bar{\mathbf W}$ is positive semi-definite and thus its spectral gap is given as $\xi_1^2 = 1- \lambda_{\bar {\mathbf W}} = \lambda_{N-1}(\bar{\mathbf W})$. We consider the  filtration introduced in Appendix~\ref{sec:conv_gossip} to upper bound the following expectation. We note that $(\mathbf z_N \mathbf z_N^\top) \mathbf h = \langle \mathbf h, \mathbf z_N\rangle \mathbf z_N = N^{-1}\mathbf 1_N $, hence:
\begin{align*}
    \mathbb{E}\left\|\left[\mathbf{W}_{0:t}^\top-\mathbf z_N \mathbf z_N^\top\right]\mathbf{h}\right\|^2 &= \mathbb{E}\left\|\mathbf{W}_{0:t}^\top\left[\mathbf{h}- \frac 1 N \mathbf 1_N\right]\right\|^2,\\
        \text{(By \cite{boyd2006randomized})} &\le\lambda_{N-1}(\bar{\mathbf W})^t\left\|\mathbf{h}- \frac 1 N \mathbf 1_N\right\|^2= \xi_t^2 \frac{b}{N(N-b)}.
\end{align*} 
Combined with $\sum_{u\in H}\|\mathbf{y}_u\|^2\le 4|H|\trinorm{\phi}^2$, we find $\mathbb{E}\|\mathbf y_H\|^2 \le 4\xi_t^2\frac b N\trinorm{\phi}^2,$ and Jensen's inequality on the square root function gives $\mathbb{E}\|\mathbf y_H\| \le 2\xi_t\sqrt{\frac b N}\,\trinorm{\phi}$ concluding the proof.
\end{proof}
\begin{lemma}
\label{lem:Qt_bound}
Given the definition of $\mathbf y_B$ in Equation~\eqref{eq:decomp}, then:
\begin{equation}\label{eq:Qt_spectral}
 \mathbb E \|\mathbf y_B\| \le 2\trinorm \phi \frac b N \zeta_t.
\end{equation}
\end{lemma}
\begin{proof}
     Let $B \subseteq \N$ such that $|B| = b$ and $H = \N\setminus B$. Let $(\mathbf z_k)_{k=1}^N$ be the eigenvectors associated with the eigenvalues $(\lambda_k(\bar{\mathbf W}))_{k=1}^N$. By the spectral theorem, $\bar{\mathbf W} = \sum_{k=1}^N \lambda_k(\bar{\mathbf W})\,\mathbf z_k\mathbf z_k^\top$. Let us denote $\mathbf 1_H, \mathbf 1_B$ respectively the binary indicator of nodes in $H$ and in $B$. Setting $\hat b_k := \langle\mathbf z_k,\mathbf{1}_B\rangle$ and $\hat h_k := \langle\mathbf z_k,\mathbf{1}_H\rangle$, we can derive the following equation. 
\begin{align}
Q_t(H) 
&:=\frac{1}{N-b}\sum_{(v,w)\in H\times B}[\bar{\mathbf W}^t ]_{vw},\nonumber\\&=\frac{1}{N-b}\langle\mathbf 1_H,\bar{\mathbf W}^t \mathbf 1_B\rangle,\nonumber\\&=\frac{1}{N-b}\sum_{k=1}^N \lambda_k(\bar{\mathbf W})^t\langle\mathbf z_k,\mathbf 1_B\rangle\langle \mathbf 1_H,\mathbf z_k\rangle,\nonumber\\
&=\frac{1}{N-b}\sum_{k=1}^N\lambda_k(\bar{\mathbf W})^t\,\hat{h}_k\hat{b}_k.\label{eq:Qt_decomp}
\end{align}
The $k=N$ terms can be computed as $\hat h_N = \langle \mathbf z_N, \mathbf 1_H\rangle = (N-b) / \sqrt N $ and $\hat b_N  = b / \sqrt N $. For $k\le N-1$, the eigenbasis is orthogonal and $\langle \mathbf z_k, \mathbf z_N\rangle = \frac 1 {\sqrt N}\langle \mathbf z_k, \mathbf 1_H\rangle + \frac 1 {\sqrt N}\langle \mathbf z_k, \mathbf 1_B\rangle = 0$ gives $\hat{b}_k=-\hat{h}_k$, so:
\begin{equation*}
Q_t(H) =  \frac{b}{N} - \frac{1}{N-b}\sum_{k=1}^{N-1}\lambda_k^t(\bar{\mathbf W})\,\hat{h}_k^2.
\end{equation*}
Using Pythagora's theorem, one can show that $\sum^{N-1}_{k=1} \hat h_k^2 =  \|\mathbf 1_H\|^2 - \hat h_N^2 = (N-b) - \hat h_N^2 = \frac{b(N-b)}{N}.$ Knowing that for all $k\in \N$, $\lambda_k(\bar{\mathbf W}) \ge \lambda_1(\bar{\mathbf W})\ge 0$ under semi-definite positiveness of $\bar{\mathbf W}$, $Q_t(H)$ is upper bounded by $\frac b N \zeta_t.$ Then, by definition of $\mathbf y_B$ and the triangle inequality:
\[
\mathbb E \|\mathbf y_B\| \le \frac {1} {N-b}\sum_{(v,w)\in H\times B}  \mathbb E [\mathbf W_{0:t}]_{vw}\|\mathbf y _w\| \le2\trinorm{\phi} Q_t(H) \le 2\trinorm \phi \frac b N \zeta_t.
\]
\end{proof}

\begin{proof}[Proof of Theorem~\ref{thm:timed_break}]
Let $\delta > 0$ and a given score with $(\phi, \psi)$-decomposition. Lemma~\ref{prop:bound_breakdown} gives that the event
$\{d_\tau\left([\psi(\bar{\mathbf x}(t))]_{\Sn}, \sigma_\star\right) \ge \delta\}$  is included in $\Omega^c(\bar{\mathbf x} (t))$.
Taking the expectation yields the standard bound:
\begin{multline}\label{eq:bound_embed_prob_r}
    \mathbb E \left\| \bar{\mathbf x}(t)  - \bar{\mathbf{x}}\right\| \ge \mathbb E [\left\| \bar{\mathbf x}(t)  - \bar{\mathbf{x}}\right\|\mathbb I \left\{d_\tau([\psi(\bar{\mathbf x}(t))]_{\Sn}, \sigma_\star)\ge 1\right\}]\\\ge \mathbb P \left(d_\tau([\psi(\bar{\mathbf x}(t))]_{\Sn}, \sigma_\star\right)\ge 1)\,r_\psi(\bar{\mathbf{x}}).
\end{multline}

Since the Kendall-$\tau$ distance is upper bounded by $\binom n 2$, we can then derive the following:
\begin{equation}\label{eq:kendall_tau_bound_prob}
\mathbb E \,d_\tau([\psi(\bar{\mathbf x}(t))]_{\Sn}, \sigma_\star) = \sum_{k=1}^{\binom n 2 } k\, \mathbb P(d_\tau([\psi(\bar{\mathbf x}(t))]_{\Sn}, \sigma_\star)=k) \le \binom n 2 \mathbb P(d_\tau([\psi(\bar{\mathbf x}(t))]_{\Sn}, \sigma_\star)\ge 1).    
\end{equation}
We consider Equation~\eqref{eq:bound_embed_prob_r}. By dividing both sides by $\binom n 2$ in  Equation~\eqref{eq:kendall_tau_bound_prob} in order to isolate the probability term and using the assumption $\mathbb E \,d_\tau([\psi(\bar{\mathbf x}(t))]_{\Sn}, \sigma_\star)\ge \delta$, we can derive:
\begin{equation}\label{eq:cond_x-x}
    \mathbb E \left\| \bar{\mathbf x}(t)  - \bar{\mathbf{x}}\right\| \ge \frac \delta{\binom n 2}r_\psi(\bar{\mathbf{x}}).
\end{equation}

Thanks to the triangle inequality, we can upper bound Equation~\eqref{eq:decomp} and take the expectation $\mathbb E \left\|\bar{\mathbf x}(t) - \bar{\mathbf{x}}_h\right\|\le\mathbb E \|\mathbf y_H\| + \mathbb E \|\mathbf y_B\|.$ 
Applying Lemma~\ref{lem:honest_bound} and Lemma~\ref{lem:Qt_bound} back to the previous decomposition yields:
\[r_\psi(\bar{\mathbf x}_h)\frac \delta{\binom n 2}\le \mathbb{E}\left\|\bar {\mathbf x}_h(t) - \bar{\mathbf{x}}_h\right\| \le  \mathbb E \|\mathbf y_H\| + \mathbb E \|\mathbf y_B\| + \left\| \bar{\mathbf x}_h  - \bar{\mathbf{x}}\right\|\le 2 \trinorm{\phi}\left(\frac b {N} (1+\zeta_t) + \sqrt{\frac{b}{N}}\xi_t\right).\]

Solving the quadratic inequality verified by $\sqrt {b/N}$ yields the following condition:
\[\varepsilon_t^H \ge \frac {\left(-\xi_t+\sqrt{\xi_t^2 +4\varepsilon_- \delta(1+\zeta_t)/\binom n 2}\right)^2}{4(1+\zeta_t)^2}.\]
\end{proof}

\section{Extensions to partial rankings}
\label{sec:practical-tie}

In many practical settings, an agent $v\in\N$ ranks only a subset of $k_v < n$ alternatives, either because the full set is too large to evaluate, or because only a portion is relevant to the agent. Formally, a partial ranking by agent $v$ is an injection $\sigma_v : \llbracket k_v \rrbracket \to \n$, where $\sigma_v(i) = a$ means that alternative $a$ is assigned rank $i$ by agent $v$. We denote by $A_v \subseteq \n$ the set of alternatives ranked by $v$, with $|A_v| = k_v$.

\paragraph{Completing partial rankings.} An approach consists of imputing the average rank $(k_v+1+n)/2$ to all unranked items \cite{hoon2004rank}, where $k_v$ is the number of ranked alternatives by node $v$. Alternatively, one can normalize ranked items as ${\sigma}(i)_{\text{scaled}} = ({ \sigma}(i)-1)/k$ and assign a normalized rank of $1$ to unranked items \cite{hoon2004rank}. Borda scores are then computed from these completed or normalized rankings in the usual way.  For Copeland's method, we use \emph{weak pairwise probabilities} $p_{ij} = \mathbb{I}\{\sigma(i)<\sigma(j)\} + (1/2)\,\mathbb{I}\{\sigma(i)=\sigma(j)\}$.

\paragraph{Modified embedding and stability radius.}
Under any of the above completion schemes, the embedding $\phi$ is applied to the completed ranking $\widetilde{\sigma}_v$ (not necessarily in $\Sn$), yielding a local state $\mathbf{x}_v(0) := \phi(\widetilde{\sigma}_v) \in \mathbb{R}^I$. We denote by $J \subseteq I$ the index set of coordinates corresponding to pairs $(i,j)$ such that at least one agent has ranked both $i$ and $j$, and restrict attention to these. We define the \emph{reduced initial dispersion} $\gamma^+ :=\sum_{i \in J}\bigl\|\mathbf{x}^{(i)}(0) - \bar{x}^{(i)}\mathbf{1}_N\bigr\|^2,$ where $\bar{x}^{(i)} = N^{-1}\sum_{v=1}^N [\phi(\widetilde{\sigma}_v)]_i$ is the population average over the completed embeddings for coordinate $i$. Similarly, since imputed coordinates are shared across agents and cancel in the averaging, the \emph{reduced stability radius} is \[r^+_\psi(\bar{\mathbf{x}}):=\inf\left\{r > 0 \,\middle|\,\exists\, \mathbf{y} \in \mathbb{R}^I, \|\mathbf{y}_{J} - \bar{\mathbf{x}}_{J}\| \leq r,d_\tau\left(\psi(\mathbf{y}),\, \psi(\bar{\mathbf{x}})\right) \geq 1\right\},\] \textit{i.e.}\ the stability radius measured only over the informative coordinate directions. Since imputed coordinates are shared across agents they cancel in the averaging, and one has $r^+_\psi(\bar{\mathbf{x}}) \leq r_\psi(\bar{\mathbf{x}})$, reflecting a degradation in ranking  identifiability relative to the complete-ranking setting.

\begin{corollary}[Convergence under partial rankings]
\label{cor:partial}
Under the same assumptions as  Theorems~\ref{thm:sync-general-ranking-convergence} and~\ref{thm:general-ranking-convergence}, the average expected  Kendall-$\tau$ error satisfies
\begin{equation}\label{eq:cor_partial_rank}\frac{1}{N}\sum_{v=1}^{N} \mathbb E\,d_\tau(\sigma_\star, \widehat{\sigma}_v(t)) \leq \frac{\binom n 2}{Nr^+_\psi(\bar{\mathbf{x}})^2}\gamma^+\rho(t),
\end{equation}
where $\rho(t) = e^{-\sqrt{\lambda_{\mathbf{W}}}t}$ under  Algorithm~\ref{alg:sync-decentralized-cons} and $\rho(t) = e^{-\lambda(p)t/2}$ under Algorithm~\ref{alg:decentralized-cons}. The bound degrades relatively to the complete ranking setting through the reduced stability radius $r^+_\psi(\bar{\mathbf{x}}) \leq r_\psi(\bar{\mathbf{x}})$.
\end{corollary}

\begin{proof}[Proof sketch]
The proof mirrors that of Theorems~\ref{thm:sync-general-ranking-convergence} and~\ref{thm:general-ranking-convergence}, with two modifications. First, in the application of Lemma~\ref{lem:kendall-err}, the Markov-inequality step linking $L^2$-error to Kendall-$\tau$ error, the summation over $I$ is restricted to $J$, and the stability radius $r_\psi(\bar{\mathbf{x}})$ is replaced by its reduced counterpart $r^+_\psi(\bar{\mathbf{x}})$, since ranking changes can only be triggered by perturbations in the informative coordinate directions. The remainder of the argument is identical, giving the stated bounds with $\gamma^+$ in place of $\gamma$. Second, the gossip averaging is applied to the completed embeddings $(\mathbf{x}_v(t))_{v \in \llbracket N \rrbracket}$, the gossip update apply verbatim in both synchronous and asynchronous cases, yielding the same exponential decay of $\|\mathbf{x}^{(i)}(t) - \bar{x}^{(i)}\mathbf{1}_N\|^2$ for every $i \in J$.
\end{proof}

In the extreme case $k_v = 1$ for all $v\in\N$ (each agent ranks a single alternative), the imputed Borda scores concentrate near $(n+1)/2$ and $r^+_\psi(\bar{\mathbf{x}}) = 0$, consistent with the loss of ranking identifiability.
\section{Further considerations for Borda consensus}
 \label{sec:further_borda}
 \subsection{Borda scoring as a median estimation} 
 
 We identify each permutation $\sigma \in \Sn$ with its vector representation in $\mathbb{R}^n$, namely $\phi_B(\sigma) := (\sigma(1), \dots, \sigma(n))^\top.$ We begin with the Spearman $\rho$ distance:
\[d_\rho(\sigma, \sigma') = \sum_{j=1}^n \big(\sigma(j) - \sigma'(j)\big)^2 = \|\phi_B(\sigma) - \phi_B(\sigma')\|^2.\]
Let $\widehat P_N$ be an empirical distribution over $\Sn$. Consider the Fréchet functional associated with $d_\rho$ as stated in Equation~\eqref{eq:metric_median},
$\mathbf s_B = \argmin_{\mathbf s \in \mathbb{R}^n} N^{-1} \sum_{v=1}^N \|[\sigma_v] - \mathbf s\|^2.$
Since the squared Euclidean norm is strictly convex, this minimization problem admits a unique solution. Taking the gradient with respect to $\mathbf s$ and setting it to zero yields the first-order condition:
\[\nabla_{\mathbf s} \left[\frac 1 N \sum_{v=1}^N\|\phi_B(\sigma_v) - \mathbf s\|^2\right] = -\frac 2 N \sum_{v=1}^N(\phi_B(\sigma_v)-\mathbf s) = 0,\]
which gives $\mathbf s_B := N^{-1}\sum_{v=1}^N\phi_B(\sigma_v).$ This procedure coincides with the first order estimator for Borda score.

\subsection{Borda stability radius.}
\label{sec:borda_radius}
We now analyze the robustness of the induced ranking for the Borda method. Let ${\mathbf{s}} \in \mathbb{R}^n$ be the score vector associated with $\widehat P_N$. Assume that all coordinates of $\mathbf s$ are distinct. Let
$(i, j) \in \arg\min_{i' \neq j'} |s_{i'} - s_{j'}|$, $\Delta := |s_{i} - s_{j}|.$ We can use an alternative definition for the stability radius:
\[r_\mathrm{Id}(\mathbf{s}) = \inf \left\{ \|\mathbf{y} - \mathbf{s}\|\mid \mathbf y \in \mathbb R^n,\, d_\tau(\mathbf{y},\mathbf{s})\ge 1\right\}.\]
A ranking changes if and only if there exists $i'\ne j'$ such that $(y_{i'} -y_{j'})(s_{i'} - s_{j'}) \le 0$, \textit{i.e.}, $\mathbf y$ crosses the hyperplane $H_{i'j'} := \{\mathbf y \in \mathbb R^n \mid y_{i'} = y_{j'}\}$. The $\ell_2$-distance from $\mathbf s$ to $H_{i'j'}$ is
\[d_2(\mathbf{s}, H_{i'j'}) = \frac{|s_{i'} - s_{j'}|}{\|\mathbf{e}_{i'} - \mathbf{e}_{j'}\|} = \frac{|s_{i'} - s_{j'}|}{\sqrt{2}}.\]
The stability radius is therefore the minimum of this quantity over all pairs $i' \neq j'$, which is achieved at $(i, j)$, and we conclude:
\begin{equation}
    r_{\mathrm{Id}}(\mathbf{s}) = \frac{\Delta}{\sqrt{2}}.
\end{equation}

\section{Further considerations for Copeland consensus}
\label{sec:further_cope}
\subsection{Copeland as a tournament-based voting rule.}
Given a multiset $\{\sigma_1, \dots \sigma_n\}\in\mathrm{Sym}^N (\Sn)$, with the associated pairwise probabilities $\widehat p_{ij} = N^{-1} \sum_{v=1}^N \mathbb I \{ \sigma_v(i) < \sigma_v(j)\}$. Copeland scores stated in Equation~\eqref{eq:copeland_scores} can be rewritten as:
\[s_i := [S_C(\widehat P_N)]_i = \sum_{\substack{j=1\\ j\ne i}}^n  \mathrm{sgn}(\widehat p_{ij} - 1/2),\]
which is generally considered as a tournament-based score, each element $i$ faces all the other contestants $j\in \n \setminus \{i\}$ and $+1$ is added to its score whenever $i$ wins over $j$ (and thus $\widehat p_{ij} > 1/2$) and $-1$ when it loses ($\widehat p_{ij} < 1/2$). The minimum score gap $\Delta = \min_{i\ne j} |s_i -s_j| \in \{0, 2\}.$  
Indeed, the difference between two scores is even. Let $i,j \in\n$ and $i\ne j$, and let us denote their respective number of wins as $w_i, w_j$ where $w_k = \sum_{j\ne k} \mathbb I \{\widehat p_{kj} > 1/2\}$ for all $k\in\n$. We know that $s_i = 2w_i - (n-1)$ and $s_j = 2w_j - (n-1)$, if we do the difference of the two we find that $2$ is a divisor of $|s_i - s_j|$ which allows us to conclude that $2$ divides $\Delta$. We can also see that the minimum score gap respects $ \Delta < 4$. Indeed, if we assume $\Delta \ge 4$, then for a given pair $(i, j)$ realizing $\Delta$, \textit{i.e.} $\Delta =|s_i -s_j|$, and without loss of generality, let us assume $s_{i} <  s_{j}$, then $j$ has two more wins than $i$. As $\Delta \ge 4$, this is true for all successive pairs of score components, and thus the total number of wins of the element with greatest score is greater than $2(n-1)$ which is absurd as they are only $n-1$ contestants. Thus, $\Delta < 4$ and $ \Delta \in\{0,2\}$.

\begin{remark}
    \label{rem:sst_copeland}
    Furthermore, if for all pairs $(i,j)$ such that $i\ne j$, $\widehat p_{ij} \ne 1/2$ and $\widehat P_N$ is strict stochastic transitive (SST), then we can ensure that $\Delta = 2$. Indeed, for a pair $(i,j)$, we can assume without loss of generality that $\widehat p_{ij} > 1/2$. Then for all $k\in\n$ such that $\widehat p_{ki} > 1/2$, then by SST, $\widehat p_{kj} > 1/2$. Thus, $j$ has at least as many losses as $i$ against other opponents and loses directly against $i$ meaning that $s^C_j < s^C_i$.
\end{remark}

\subsection{Copeland score stability radius}
\label{sec:cope_radius}
We now analyze the robustness of the induced ranking for the Copeland method. Let ${\mathbf{s}} \in \mathbb{R}^n$ be the score vector associated with $P$. Assume that all coordinates of $\mathbf s$ are distinct.  Let $(i, j) \in \arg\min_{i' < j'} |s_{i'} - s_{j'}|$, $\Delta := |s_{i} - s_{j}|.$ Without loss of generality, assume that $ p_{ij} < 1/2$, and thus $s_{i}  < s_{j}$. In order to swap scores, we need to add $|p_{ij} - 1/2| $ to $p_{ij}$. If we denote $\mathbf s'$ the corrupted score vector, then $s'_{i} = s_{i} + 2$ and $s_{j}' = s_{j} - 2.$ A score swap occurs when $s'_{i} - s'_{j} >0$ and:
\[s'_{i} - s'_{j} = (s'_{i} - s_{i}) - (s'_{j} - s_{j}) + (s_{i} - s_{j}) =  -\Delta + 4, \]
    combining the two expressions, we find that a score swap occurs when $\Delta < 4$. Using the result from previous subsection and the following Remark~\ref{rem:sst_copeland} allows us to conclude that under strict stochastic transitivity:
\begin{equation}
    r_{\psi_C}(\bar{\mathbf x}) = \min_{i< j} | p_{ij} - 1/2|.
\end{equation}
We can of course extend the previous result to an empirical distribution $\widehat P_N$ as it is the case in the decentralized estimation.

\begin{remark}
    Copeland scoring can be extended to a Copeland-$\alpha$ scoring function \cite{faliszewski_copeland_2008} for $\beta\in[0,1]$. While the embedding map is share with the original Copeland method, the aggregation map $\psi_{C,\beta}$ is given for all $i\in \n$ by
    \[
        [\psi_{C,\beta}(\mathbf x)]_i = \sum_{j\ne i} \mathbb I \left\{ \mathbf x_{ij} > \frac 1 2\right\} - \mathbb I \left\{ \mathbf x_{ij} < \frac 1 2\right\} + (2\beta -1 ) \mathbb I \left\{\mathbf x_{ij} = \frac 1 2\right\}.
    \]
    In fact, for $\alpha=1/2$, we retrieve the original Copeland voting rule described in Table~\ref{tab:ranking_methods}. Considering $\mathbf x$ outside of the hyperplane $H_{ij} := \{\mathbf x \in [0,1]^{\binom n 2} \mid \mathbf x_{ij}= 1/2\}$ for $i < j$, it is easy to see that $\psi_{C,\beta}(\mathbf x) = \psi_C(\mathbf x)$ for all $\beta \in[0,1]$. Since the plane has an empty interior, it forces the cell boundaries $\partial C(\sigma_\star)$  to be equal when no pairwise ties are present (see Remark~\ref{rem:geom_interp}). When $\bar {\mathbf x}\in H_{ij}$ then the distance to the cell is null. 
\end{remark}

\section{Likelihood estimation on $\Sn$}
\label{sec:prob_model}
In this section, we introduce parametric probabilistic models and discuss how the previous constructions adapt to maximum likelihood estimation (MLE). Indeed, for the model shown below finding the consensus $\sigma_\star$ given the method of consensus often translates into finding the MLE. We discuss two common parametric models on $\Sn$.

\subsection{Mallows model}
\paragraph{Setup.} We consider a right-invariant metric $d$ (\textit{e.g.} \ Kendall-$\tau$ $d_\tau$) on $\Sn$, \textit{i.e.} for $\sigma, \sigma'\in \Sn$ $d(\sigma', \sigma) = d(\sigma' \circ \sigma^{-1}, \mathrm{Id})$, where $\mathrm{Id}$ denotes the identity permutation in $\Sn$ (See Chap. 6 \cite{diaconis_group_1988}). Mallows model $P^M_\theta$ \cite{Mallows57} is a parametric distribution on $\Sn$, with parameter $\theta = (\sigma_\star, \varphi)$, where $\varphi\in [0,1]$ is the dispersion parameter and $\sigma_\star\in \Sn$ the central permutation, such that for $\sigma\in\Sn$:
\[P^M_\theta(\sigma ) = Z(\varphi)^{-1}\varphi^{d(\sigma, \sigma_\star)},\]
where $Z(\varphi)$ is the normalizing constant associated with the distribution, the constant depends only on the dispersion parameter thanks to the right-invariant property of $d$:

\[\sum_{\sigma\in \Sn} \varphi^{d(\sigma, \sigma_\star)} = \sum_{\sigma\in \Sn} \varphi^{d(\sigma_\star \circ\, \sigma^{-1}, \mathrm{Id})} = \sum_{\sigma'\in\Sn} \varphi^{d(\sigma', \mathrm{Id})} =: Z(\varphi).\]
When $\varphi \to 0$, the distribution tends to the Dirac distribution $\delta_{\sigma_\star}$. When $\varphi\to 1$, the probability distribution is the Haar measure on $\Sn$ \textit{i.e.}, for all $\sigma\in\Sn$, $P_{\mathrm{Haar}}(\sigma) = (n!)^{-1}$. The interest of such probability distribution is that it admits a mode given by $\sigma_\star$ as well as symmetry properties with respect to the distance $d$. A \emph{mixture of Mallows distributions} is defined, for parameters $\theta_1, \dots, \theta_k$ and a weight vector $ (\pi_i)_{i=1}^k$ with $\sum_{i=1}^k \pi_i = 1$, as:
\[P^M(\sigma) = \sum_{i=1}^k \pi_i P^M_{\theta_i}(\sigma).\]

We quickly discuss the process to sample Mallows distributed samples for the synthetic dataset experiments. To sample permutations following $P^M_\theta$, we used the RIM algorithm or variants (See \cite{doignon_repeated_2004, lu_effective_2014}). The idea is to switch the elements of the central parameter $\sigma_\star$, with acceptance $\varphi^k$ with $k\in\n$ being the inversion length.

We aim to estimate the central parameter, the log-likelihood of $P^M_\theta$ is as follows:
\[\ell (\theta) = -\log(1/\varphi)\sum_{v=1}^N d(\sigma_\star, \sigma_v)- N\log(Z(\varphi)).\]
Thus, the MLE $\widehat \sigma$ for $\sigma_\star$ is 
\[\widehat \sigma \in\argmin_{\sigma\in\Sn} \sum_{v=1}^N d(\sigma, \sigma_v).\]
We can see that finding the MLE is the same as solving the problem of Equation~\eqref{eq:metric_median}. Using the Corollary 3 from \cite{busa-fekete_preference-based_2014}, we know that the pairwise probability for a Mallows distribution with dispersion parameter $\varphi$ follows the conditions given in Lemma~\ref{lem:busa fekete}.
\begin{lemma}[\citet{busa-fekete_preference-based_2014}]
\label{lem:busa fekete}
For Mallows model $P^M_\theta$ with parameter $\theta = (\sigma_\star, \varphi)$, and a given pair $(i,j)$ with $i\ne j$, the following properties hold:
\begin{enumerate}
    \item $p_{ij} \ge \frac 1 {1+ \varphi} > \frac 1 2 \implies \sigma_\star(i) < \sigma_\star(j),$ with equality when $\sigma_\star(i) = \sigma_\star(j)-1 $.
    \item $p_{ij} \le \frac \varphi {1+ \varphi} < \frac 1 2 \implies \sigma_\star(i) > \sigma_\star(j),$ with equality when $\sigma_\star(j) = \sigma_\star(i)-1 $.
    \item $p_{ij}>\frac 1 2$ iff $\sigma_\star(i) <\sigma_\star(j) $ and $p_{ij}<\frac 1 2$ iff $\sigma_\star(i) > \sigma_\star(j)$.
\end{enumerate}
\end{lemma}

\paragraph{Derivation of the Borda stability radius.}
An analogous characterization holds for the Borda method. Under $P^M_\theta$, the expected Borda score of item $i$ is $s^{B}_i = \mathbb{E}_{P^M_\theta}[\sigma(i)]$, which depends on both the position $\sigma_\star(i)$ and the dispersion $\varphi$. The minimum score gap $\Delta_{B}(\varphi) := \min_{i \neq j}|s^B_i - s^B_j|$ governs the Borda stability radius via $r_{\mathrm{Id}}(\bar{x}) = \Delta_B(\varphi)/\sqrt{2}$ (see Appendix~\ref{sec:borda_radius}). Two limiting regimes clarify its behavior. As $\varphi \to 0$, the distribution concentrates on $\sigma_\star$, so $s^B_i \to \sigma_\star(i)$ and consecutive score gaps approach $1$, giving $r_{\mathrm{Id}} \to 1/\sqrt{2}$. As $\varphi \to 1$, all Borda scores collapse toward the common value $(n+1)/2$ under the Haar measure, so $\Delta_B(\varphi) \to 0$ and $r_{\mathrm{Id}} \to 0$. In summary:
\[r_{\mathrm{Id}}(\bar{\mathbf x}) = \frac{\Delta_{B}(\varphi)}{\sqrt{2}}, \qquad \text{with } \Delta_{B}(\varphi) \xrightarrow[\varphi \to 0]{} 1 \quad\text{and}\quad \Delta_{B}(\varphi) \xrightarrow[\varphi \to 1]{} 0.\]

\paragraph{Derivation of the Copeland stability radius.}
From Lemma~\ref{lem:busa fekete}, for any pair $(i,j)$ with $\sigma_\star(i) < \sigma_\star(j)$, we have
$p_{ij} \geq \frac{1}{1+\varphi}$, so that
\[\left|p_{ij} - \tfrac{1}{2}\right| \;\geq\; \frac{1}{1+\varphi} - \frac{1}{2} = \frac{1-\varphi}{2(1+\varphi)},\]
with equality precisely when $\sigma_\star(i) = \sigma_\star(j)-1$, \textit{i.e.}, for adjacent pairs in $\sigma_\star$. Since Appendix~\ref{sec:cope_radius} establishes that $r_{\psi_C}(\bar{x}) = \min_{i < j}|p_{ij} - 1/2|$ under strict stochastic transitivity, taking the minimum over all pairs yields the claimed expression:
\[r_{\psi_C}\left(\mathbb{E}_{P^M_\theta} \phi_C\right) = \frac{1-\varphi}{2(1+\varphi)}.\]
This quantity is strictly positive for $\varphi \in [0,1)$ and tends to zero as $\varphi \to 1$ (\textit{i.e.}, as the distribution approaches the Haar measure), reflecting the fact that the consensus ranking becomes unidentifiable in the uniform regime.

\paragraph{Comparison of breakdown thresholds.}\label{comp:rad}
Substituting these stability radii into the general lower bound of Proposition~\ref{prop:bound_breakdown}, and using the mixed norms $\trinorm{\phi_{C}} = \sqrt{n(n-1)/2}$ and $\trinorm{\phi_{B}} = \sqrt{n(n+1)(2n+1)/6}$, the method-specific breakdown thresholds under a Mallows model with dispersion factor $\varphi$ read:
\[\varepsilon^-_{C} = \frac{1-\varphi}{2(1+\varphi)\sqrt{2n(n-1)}}, \qquad \varepsilon^-_{B} = \frac{\Delta_{B}(\varphi)}{2 \sqrt{n(n+1)(2n+1)/3}}.\]
The ratio $r_\psi(\bar{x})/\trinorm{\phi}$, which controls the breakdown threshold, grows as $O(n^{-\beta})$ and $\beta =1$ for Copeland and $\beta = 3/2$ for Borda as $n \to \infty$.  We note however that both thresholds vanish as $\varphi \to 1$, consistently with the loss of identifiability of $\sigma_\star$ under near-uniform preferences.

\section{Approximated metric-based consensus.}
\label{sec:gossip-other}
We are interested in problems that follow a metric-based minimization problem as described in \eqref{eq:metric_median}. A correspondence can be made between barycentric (metric-based) and score-based approaches. Indeed, minimizing the sum of Spearman-$\rho$ correlations $d_\rho({ \sigma}, { \sigma}^{\prime})=\sum_{i=1}^n({\sigma}(i)-{ \sigma}^{\prime}(i))^2$ leads to the Borda consensus. Spearman's footrule $d_1(\sigma, \sigma') := \sum_{i=1}^n |\sigma(i) - \sigma'(i)|$, in turn, gives rise to the median aggregation rule \cite{fagin2003efficient}. If the coordinate-wise median ranks across voters form a valid permutation, then this permutation minimizes the total footrule distance \cite{fagin2003efficient}. We first describe techniques to make the computations of specific medians more tractable, before giving specific details about a greedy approach for Kemeny consensus, we first describe heuristic methods to estimate the minimum of a median estimation problem for Spearman Footrule distance $d_1$ and the use of Local Kemenization to enforce robustness with respect to $d_\tau$.

\subsection{Spearman Footrule relaxation}
In this section, we present a gossip-based method for the Footrule consensus (see Alg. \ref{alg:footrule}) using Goal-PD for decentralized median estimation \cite{van2026fast}. We also introduce a decentralized implementation of Local Kemenization (see Alg. \ref{alg:local-kem}), which enables closer approximation to the Kemeny consensus while remaining computationally tractable.

\begin{algorithm}[ht]
\caption{Decentralized Footrule Consensus \cite{fagin2003efficient}.}
\label{alg:footrule}
\begin{algorithmic}[1]
\STATE \textbf{Init}: Each voter $v$ initializes its estimate as $\mathbf{x}_v = \sigma_v$.\\ Initialize median estimator: $\texttt{goalpd.init()}$. 
\FOR{$t=1, 2, \ldots, T$}
\STATE Select an edge $e = (u,v)\in E$ with probability $p_e$.
\STATE $\mathbf{x}_u, \mathbf{x}_v \leftarrow \texttt{goalpd.update}({\mathbf{x}_u,  \mathbf{x}_v})$.
\ENDFOR 
\STATE \textbf{Output:} For all node $v\in\N$, return the Footrule (median rank) estimate $\widehat{ \sigma}_v$ by sorting the components of $\mathbf{x}_v$ in ascending order.
\end{algorithmic}
\end{algorithm}
\subsection{Local Kemenization}
Beyond global optimality, local Kemeny optimality \cite{dwork2001rank} offers a computationally tractable alternative that satisfies the extended Condorcet property, ensuring that Condorcet winners are ranked highly while Condorcet losers are placed near the bottom. Every Kemeny optimal solution is locally Kemeny optimal, but the converse does not hold, making local optimality an attractive relaxation that approximates the Kemeny consensus. Any initial ranking can be transformed into a locally Kemeny optimal solution via the following iterative procedure, known as local Kemenization \cite{dwork2001rank}: \textit{(i)} initialize with an aggregation from a tractable method, such as Borda or Copeland, \textit{(ii)} For each pair of adjacent candidates $(i, j)$ in the current ranking, check whether a majority of voters prefer $j$ over $i$ (\textit{i.e.}, $\hat{p}_{ij} < 1/2$). If so, swap $i$ and $j$, \textit{(iii)} Repeat step b until no further swaps occur. This procedure converges in polynomial time and inherits robustness properties against \textit{spam} or noisy inputs.\\
\begin{algorithm}[ht]
\caption{Decentralized Local Kemenization \cite{dwork2001rank}.}
\label{alg:local-kem}
\begin{algorithmic}[1]
\STATE Select aggregation method $\texttt{consensus}$  (\textit{e.g.}, Copeland, Borda).
\STATE \textbf{Init}: Each voter $v$ initializes its estimate as $\mathbf{x}_v = \left(\widehat p^v_{i j}\right)_{i < j}$, where $\widehat p^v_{ij}=\mathbb{I}\!\left[\sigma_v(i) < \sigma_v(j)\right]$. Do $\texttt{consensus.init()}$.
\FOR{$t=1, 2, \ldots, T$}
\STATE Select an edge $e = (u,v)\in E$ with probability $p_e$.
\STATE Update pairwise scores: $\mathbf{x}_u, \mathbf{x}_v \leftarrow ({\mathbf{x}_u + \mathbf{x}_v})/2$.
\STATE Update consensus scores: \texttt{consensus.update()}.
\ENDFOR
\STATE Compute preliminary ranking: $\widehat{ \sigma}_v\gets \texttt{consensus.sort()}$. Apply Local Kemenization on $\widehat{ \sigma}_v$ using $x_{v,(i,j)}$ in place of $\widehat{p}_{ij}$, following the procedure described above.
\STATE \textbf{Output:} For all nodes $v\in\N$, return the locally Kemeny optimal aggregation $\widehat{ \sigma}_v$.
\end{algorithmic}
\end{algorithm}

We denote by $L_P(\sigma) := \mathbb E_{\Sigma\sim P}\, d_\tau(\sigma, \Sigma)$ the Kemeny loss for distribution $P$. If $\sigma_\star$ is a solution of Equation~\eqref{eq:metric_median} with $P = \widehat P_N$ and $d = d_\tau$, the optimal loss is denoted by $L^\star_{\widehat P_N} :=L_{\widehat P_N}(\sigma_\star)$. The excess risk can be defined for a given consensus $\sigma\in\Sn$ as:
\begin{equation}\label{eq:excess_risk_kem}
    \Delta L_{\widehat P_N}(\sigma) = L_{\widehat P_N}(\sigma) - L^\star_{\widehat P_N}.
\end{equation}

This method stems from the fact that the Kendall-$\tau$ distance is minimized when the minimal number of inversions occurs.

\section{Other experiments}
\label{sec:other_exp}
In this section, we present several experiments not included in Section~\ref{sec:numerical_exp} notably on synthetic data models especially with probabilistic models as detail in Appendix~\ref{sec:prob_model}, as well as considerations for metric-based consensus as mentioned in Appendix~\ref{sec:gossip-other}. As in Section~\ref{sec:numerical_exp}, for all settings, the gossip communication follows a uniform distribution over edges, \textit{i.e.}, $p = (1/|E|)_{e \in E}$
\subsection{Asynchronous convergence study on synthetic data}\label{sec:synth-data-conv}
\paragraph{Experimental setup.} 

We generate synthetic ranking data according to the Mallows model $P^M_\theta$ with parameter $\theta= (\sigma_\star, \varphi)$. We set the number of agents to $N = 151$, the number of items to $n = 8$, and the dispersion parameter to $\varphi = 0.5$, with a central permutation $\sigma_\star$ drawn uniformly at random from $\Sn$ and fixed across all trials. In each trial, the rankings are sampled independently from the Mallows distribution with parameters $\theta$. Communication graphs are generated using three standard random graph models: complete graphs, Watts–Strogatz graphs, and random geometric graphs. For geometric graphs we set the connectivity radius $r=\sqrt{(\log(N) +16)/(N\pi)}$, and for Watts–Strogatz graphs we set $k= \max(\lfloor N/500\rfloor, 6)$. All generated graphs are verified to have a connectivity constant $\lambda(p) \in (0,1)$, ensuring they are connected and non-bipartite, as required by standard convergence assumptions. We report results averaged over $100$ independent trials, with each trial running for $2,000$ iterations. To apply the Goal-PD algorithm, we consider the step size $\rho =1$ and refer to \cite{van2026fast} for more details on their significance. \\

\begin{figure*}[ht]
    \centering
    \begin{subfigure}{0.32\textwidth}
        \centering
        \includegraphics[height=\myheight, width=\mywidth]{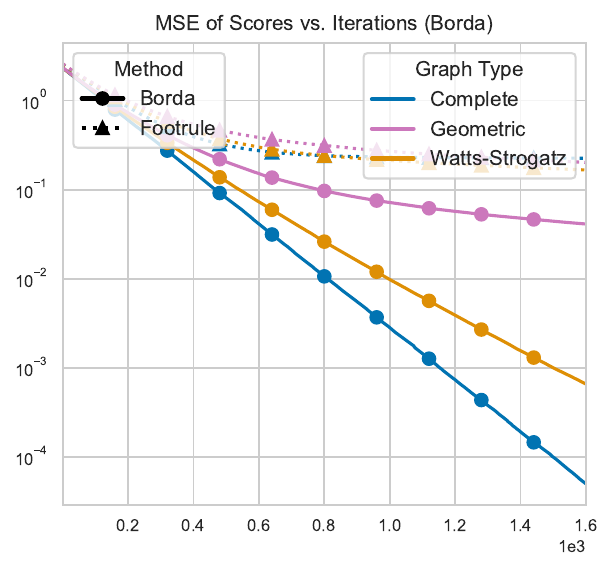}
        \caption{MSE of Borda and Footrule scores}
    \end{subfigure}\hfill
    \begin{subfigure}{0.32\textwidth}
        \centering
        \includegraphics[height=\myheight, width=\mywidth]{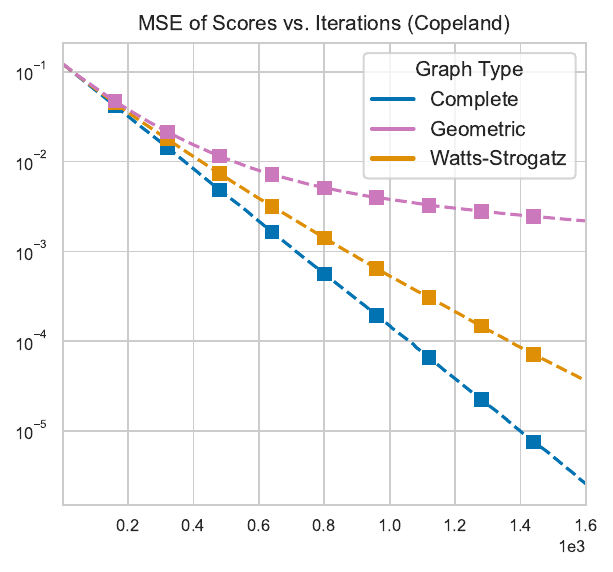}
        \caption{MSE of pairwise probabilities (Copeland's mathod)}
    \end{subfigure}\hfill
    \begin{subfigure}{0.32\textwidth}
        \centering
        \includegraphics[height=\myheight, width=\mywidth]{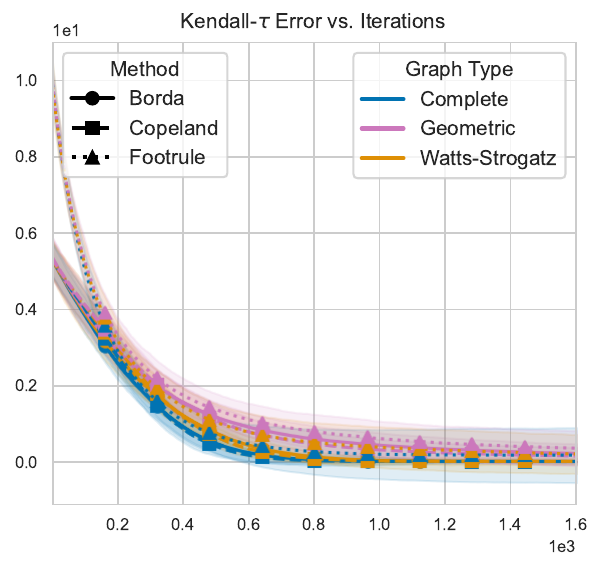}
        \caption{Average $d_\tau$ error to respective consensus}
    \end{subfigure}
    \caption{Convergence of consensus methods (Borda, Copeland, Footrule) on rankings sampled from a Mallows model ($N=151$ agents, $n=8$ items, $\varphi=0.5$) across different graph topologies. Results show mean $\pm$ standard deviation over $100$ trials.}
    \label{fig:mallow_res}
\end{figure*}

\noindent\textbf{Results.} Figure~\ref{fig:mallow_res} presents the convergence behavior of the proposed gossip-based decentralized consensus methods under the Mallows model. The first plot shows the evolution of the mean squared error (MSE) of the Borda scores (\textit{i.e.}, mean of ranks) and Footrule scores (\textit{i.e.}, median of ranks), averaged over all node estimates relative to their true scores. Both methods converge rapidly across all graph topologies, achieving an MSE below $0.5$ in fewer than $1000$ iterations. The Footrule score converges slightly more slowly than the Borda score, which is expected given that estimating the median is inherently more challenging than estimating the mean. The second plot shows the convergence of the pairwise scores underlying the Copeland consensus, which similarly shows fast and stable convergence. The third plot reports the average Kendall-$\tau$ distance between the estimated consensus rankings (Borda, Copeland, and Footrule) and their respective ground-truth consensus rankings. Convergence is again rapid across all methods, with the complete graph converging fastest, followed by the Watts–Strogatz and geometric graphs, consistent with their respective degrees of connectivity.\\

\subsection{Comparison of convergence rate between the synchronous and asynchronous versions}\label{sec:comparison_rates}

\paragraph{Experimental setup.}
We consider a setup with $N=250$ voters and $m=7$ alternatives. In this experiment, we compare the communication costs of Algorithm~\ref{alg:sync-decentralized-cons} and Algorithm~\ref{alg:decentralized-cons}. To do so, we compare the empirical estimations of $T^{\mathrm{Rand}}(0)$ and $T^{\mathrm{Sync}}(0)$ on $30$ trials. To make the comparison more accurate, we also consider the number of communication \textit{i.e.} the number of times a node calls its neighbors in the synchronous setting, in order to compare it with the randomized setting and the value of $T^{\mathrm{Rand}}(0)$. 

\begin{table}[ht]
\centering
\begin{tabular}{c|c|ccc}
\toprule

 Scoring rule & Graph & $T^{\mathrm{Rand}}(0)$ & $T^{\mathrm{Sync}}(0)$ &  Num. of communications \\
\midrule
Borda& 2D Grid & 8060.0 & 47 & 21855\\\
 &Geometric & 8013 & 28 & 41244 \\
 & Watts-Strogatz & 1880 & 5 & 6250 \\\midrule
 Copeland & 2D Grid & 6794 & 41 & 19065 \\
 &Geometric & 6549 & 24 & 35352 \\
 & Watts-Strogatz & 1677 & 4 & 5025 \\
\bottomrule
\end{tabular}
\vspace{0.3em}
\caption{Comparison of mixing times in the synchronous and randomized settings.}
\label{tab:mixing_times}
\end{table}

\paragraph{Results.} In Table~\ref{tab:mixing_times}, we present the mixing times and total number of communications. While, the number of time steps required for the complete mixing of the rankings in the randomized setting is usually significantly higher that $T^{\mathrm{Sync}}(0)$, it is in fact comparable and usually lower than the total number of communication required in the synchronous case. 

\subsection{Convergence study on Mallows model with varying $\varphi$}\label{sec:ablation_mallows}
\paragraph{Experimental setup.}
We consider a setup with $N = 1000$ voters and $n = 7$ alternatives.
Preferences are drawn from a Mallows model with a fixed central ranking $\sigma_\star = \mathrm{Id}$ and dispersion parameter $\varphi \in \{0.3, 0.5, 0.7, 0.9\}$, where small values of $\varphi$ correspond to rankings concentrated near $\sigma_\star$ (low noise) and large values correspond to near-uniform preferences (high noise). We evaluate on the four graph topologies previously considered to govern the gossip communication structure, all graphs are configured to have $N = 1000$ nodes. We measure performance along two axes: the MSE between each agent's estimated Borda (resp.\ Copeland) score vector and the ground-truth score vector induced by $\sigma_\star$, and the Kendall-$\tau$ distance between the consensus ranking decoded from the aggregated scores and the central ranking $\sigma_\star$.
The gossip process is run for up to $10{,}000$ iterations, with metrics recorded at logarithmically spaced checkpoints.
\paragraph{Results.}
The results are presented in Figure~\ref{fig:disp_mallows_plots}. Across all values of $\varphi$, graph topology is the dominant factor governing convergence: the Complete graph achieves the fastest decay in both MSE and Kendall-$\tau$ error, while the 2D Grid is consistently the slowest, with Watts-Strogatz and Geometric graphs occupying intermediate positions. As $\varphi$ increases from $0.3$ to $0.9$, preferences become more dispersed and convergence degrades for all topologies, with the 2D Grid suffering most severely, at $\varphi = 0.9$, its Kendall-$\tau$ error remains well above zero after $10000$ iterations while the Complete graph still converges to near zero. Borda and Copeland exhibit nearly identical ranking recovery across all settings, indicating that the convergence bottleneck lies in the gossip mixing time of the graph rather than the choice of aggregation rule.
\begin{figure*}[t]
\centering
\setlength{\tabcolsep}{2pt}
\begin{tabular}{ccccc}
\subcaptionbox{Borda MSE, $\varphi =0.3$ (low noise)\label{fig:a1}}{\includegraphics[width=0.24\linewidth]{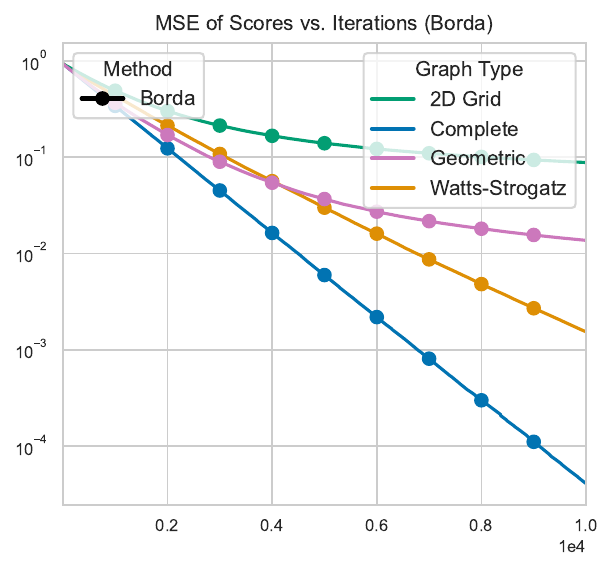}} &
\subcaptionbox{Borda MSE, $\varphi =0.5$\label{fig:b1}}{\includegraphics[width=0.24\linewidth]{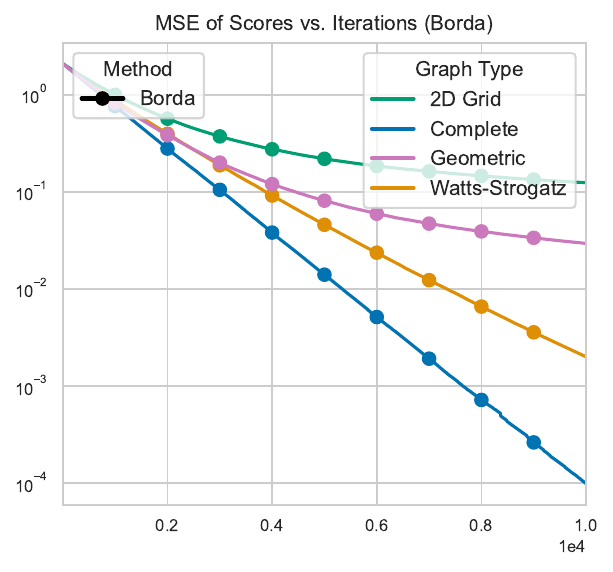}} &
\subcaptionbox{Borda MSE, $\varphi =0.7$\label{fig:c1}}{\includegraphics[width=0.24\linewidth]{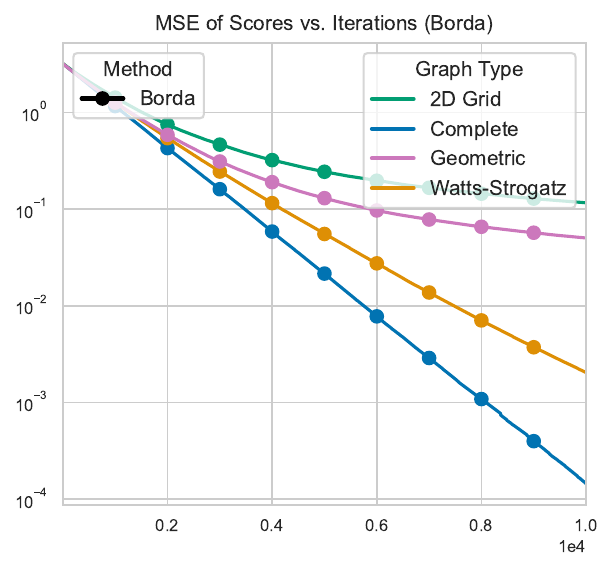}} &
\subcaptionbox{Borda MSE, $\varphi =0.9$ (high noise)\label{fig:e1}}{\includegraphics[width=0.24\linewidth]{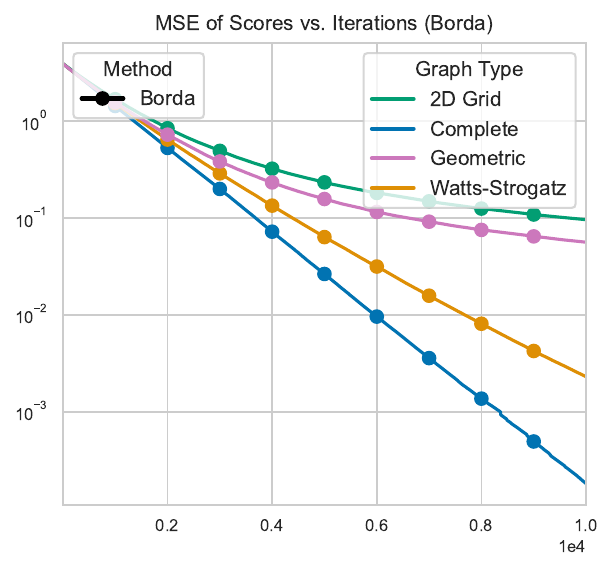}} \\
\subcaptionbox{Copeland MSE, $\varphi =0.3$ (low noise)\label{fig:a2}}{\includegraphics[width=0.24\linewidth]{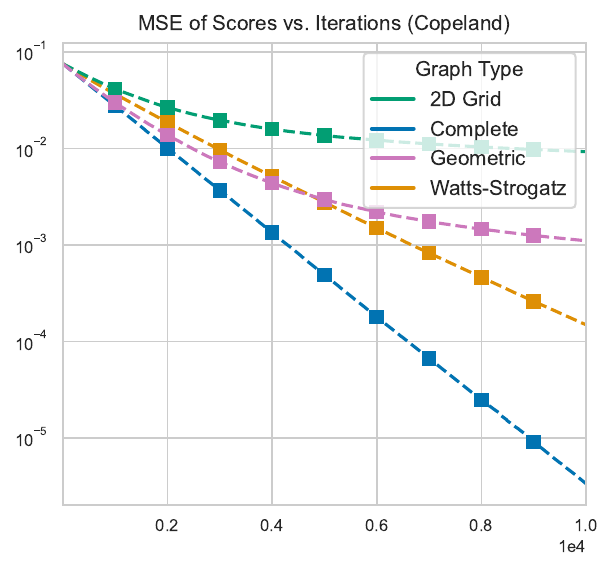}} &
\subcaptionbox{Copeland MSE, $\varphi =0.5$\label{fig:b2}}{\includegraphics[width=0.24\linewidth]{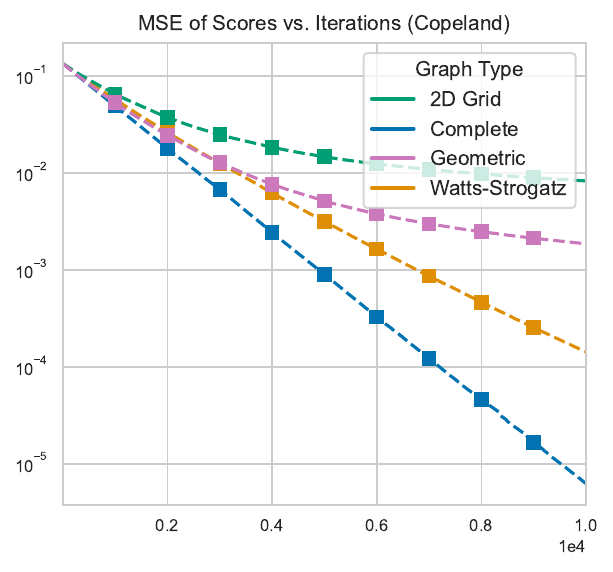}} &
\subcaptionbox{Copeland MSE, $\varphi =0.7$\label{fig:c2}}{\includegraphics[width=0.24\linewidth]{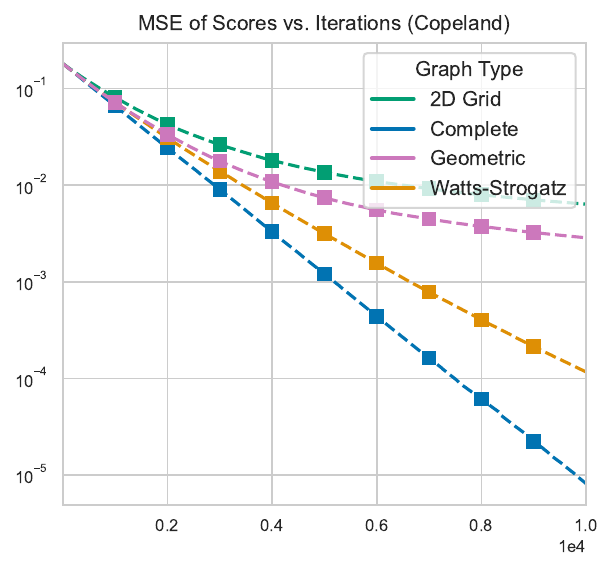}} &
\subcaptionbox{Copeland MSE, $\varphi =0.9$ (high noise)\label{fig:d2}}{\includegraphics[width=0.24\linewidth]{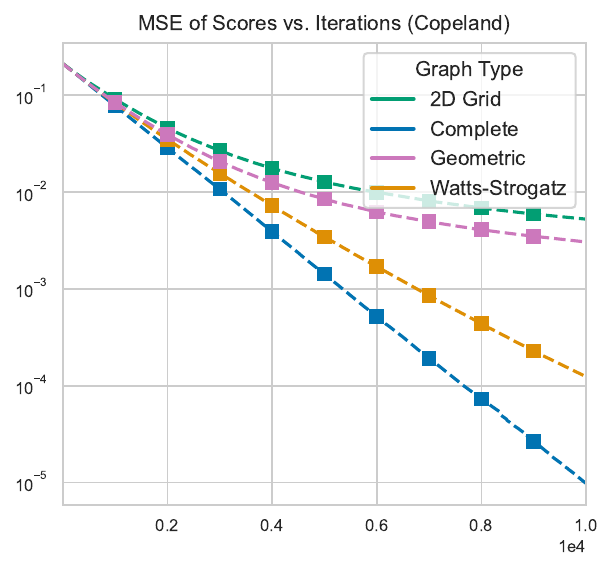}}\\
\subcaptionbox{Borda vs.\ Copeland consensus, $\varphi =0.3$\label{fig:a3}}{\includegraphics[width=0.24\linewidth]{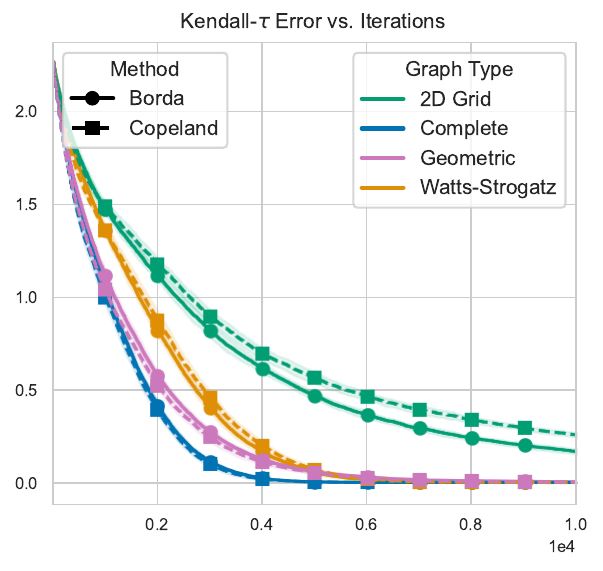}} &
\subcaptionbox{Borda vs.\ Copeland consensus, $\varphi =0.5$\label{fig:b3}}{\includegraphics[width=0.24\linewidth]{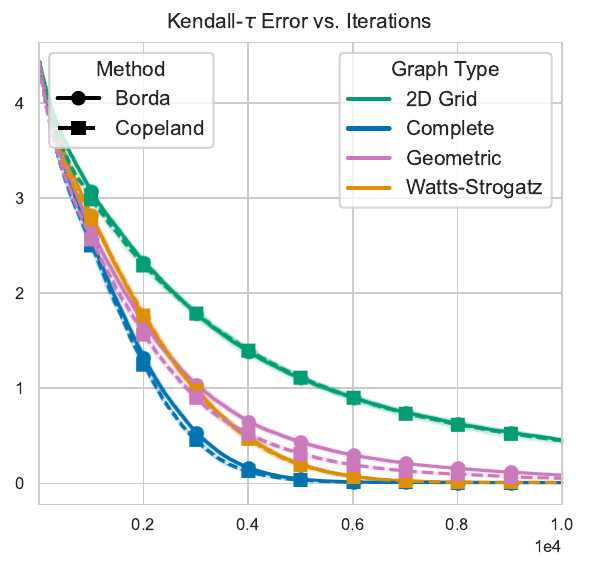}} &
\subcaptionbox{Borda vs.\ Copeland consensus, $\varphi =0.7$\label{fig:c3}}{\includegraphics[width=0.24\linewidth]{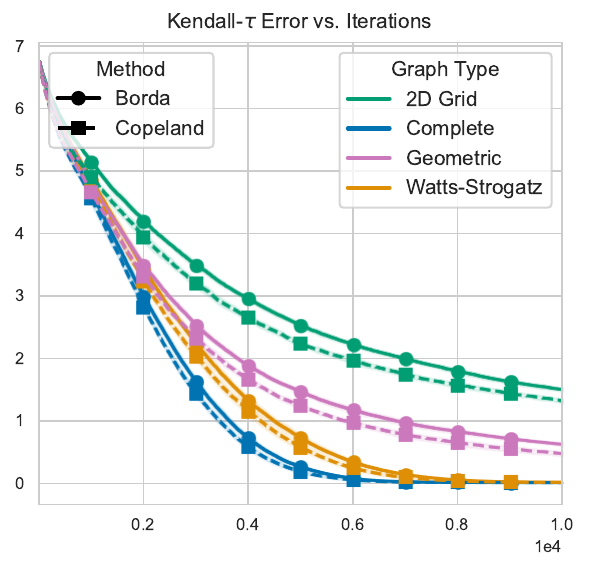}} &
\subcaptionbox{Borda vs.\ Copeland consensus, $\varphi =0.9$\label{fig:d3}}{\includegraphics[width=0.24\linewidth]{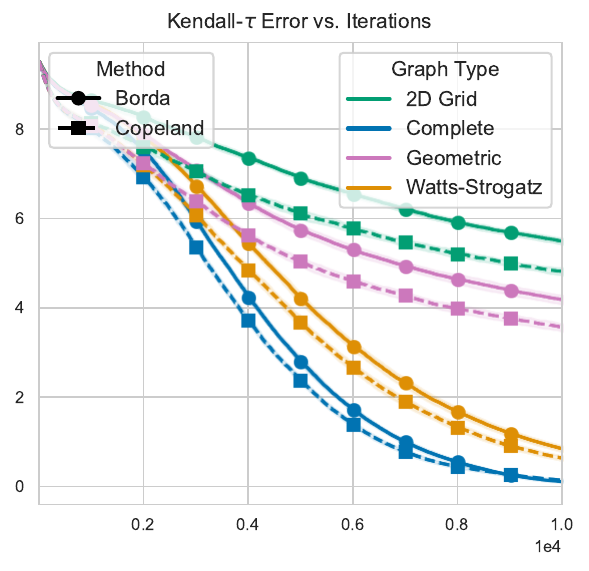}} \\
\end{tabular}
\caption{Score estimation error and consensus quality under the Mallows model ($N=1000$ voters, $n=7$ alternatives) across four dispersion levels $\varphi \in \{0.3, 0.5, 0.7, 0.9\}$ (columns, left to right). Small $\varphi$ indicates low noise (rankings concentrated near the central order); large $\varphi$ indicates high noise (near-uniform preferences). Rows 1-2 show the MSE of estimated Borda and Copeland scores, respectively, as a function of the number of sampled rankings. Row 3 compares the consensus rankings produced by the two rules, measured by their distance to the ground-truth central ranking.}
\label{fig:disp_mallows_plots}
\end{figure*}

\subsection{Empirical breakdown function sweep for different $n$ and $N$}\label{sec:sweep_n_N}

\paragraph{Experimental setup.}
We follow the experimental setup of Section~\ref{sec:breakdown_expe}, and sweep the empirical breakdown function $\varepsilon(\delta)$ for different values of the number of alternatives $n\in \{5, 7,10, 15\}$ and $N\in\{100, 250, 500\}$ for both Borda and Copeland methods. We fix $\delta=1$ in the results to show the minimum contaminated fraction required to move the ranking of at least $1$ in Kendall-$\tau$ distance with respect to the original consensus ranking.
\paragraph{Results.}
Following the results of Table~\ref{tab:break_sweep}, 
the empirical breakdown function increases with $N$, as expected from the topology of the graph. For smaller values of $N$, for which the honest signal is noisier, only a small contamination fraction is needed to break the consensus. However, the contamination required increases once $N$ grows large relatively to $n$.
\begin{table}[ht]
\centering
\begin{tabular}{c ccc ccc}
\toprule
 & \multicolumn{3}{c}{$\varepsilon_B(1)$} & \multicolumn{3}{c}{$\varepsilon_C(1)$} \\
\cmidrule(lr){2-4} \cmidrule(lr){5-7}
$n / N$ & 100 & 250 & 500 & 100 & 250 & 500 \\
\midrule
5  & 0.010 & 0.065 & 0.120 & 0.010 & 0.065 & 0.102 \\
7  & 0.010 & 0.047 & 0.120 & 0.010 & 0.047 & 0.120 \\
10 & 0.010 & 0.010 & 0.102 & 0.010 & 0.065 & 0.083 \\
15 & 0.010 & 0.010 & 0.157 & 0.010 & 0.047 & 0.083 \\
\bottomrule
\end{tabular}
\vspace{0.3em}
\caption{Empirical breakdown function valued at $t=7500$ and $\delta=1$ for both Borda and Copeland methods.}
\label{tab:break_sweep}
\end{table}

\subsection{Breakdown study with a stubborn model}
\paragraph{Setup.} We replicate the same setup as the one from Section~\ref{sec:breakdown_expe}, but we consider another methodology regarding the contaminated nodes. Indeed, we consider in this setup that the adversarial nodes contaminates the honest ones, without participating in the communication protocol, hence we refer to the adversarial nodes as 'stubborn'. 

In Figure~\ref{fig:breakdown_stubborn}, the results are similar to the ones of Figure~\ref{fig:breakdown}, as the trends of the empirical breakdown functions for both Borda and Copeland consensus are closely related.

\begin{figure*}[ht]
    \centering
    \begin{subfigure}{0.33\textwidth}
        \centering
        \includegraphics[height=\fourheight]{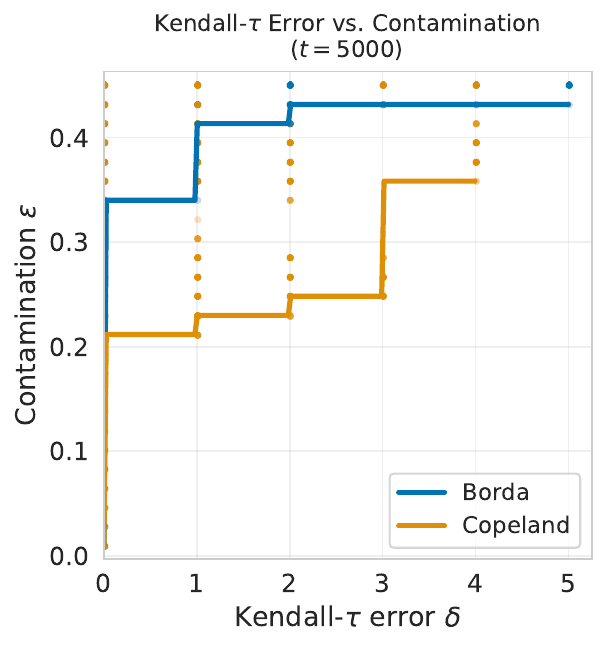}
        \caption{At iteration $t=5000$}
    \end{subfigure}\hfill
    \begin{subfigure}{0.33\textwidth}
        \centering\includegraphics[height=\fourheight]{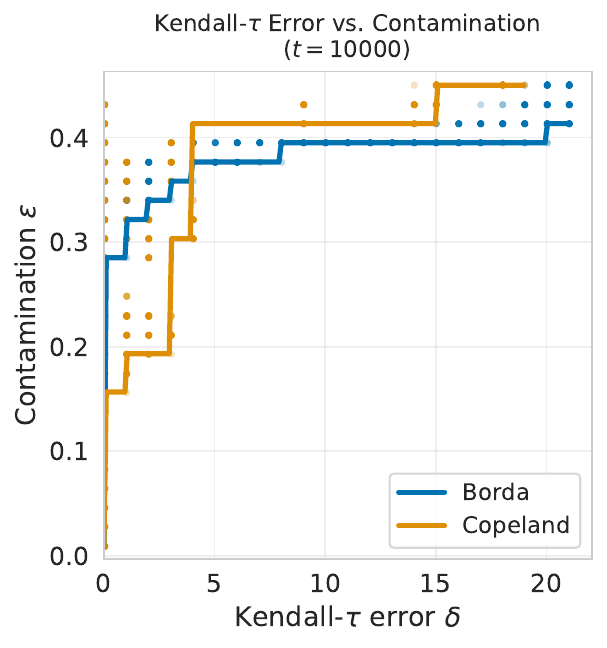}
        \caption{At iteration $t=10000$}
    \end{subfigure}\hfill
    \begin{subfigure}{0.33\textwidth}
        \centering\includegraphics[height=\fourheight]{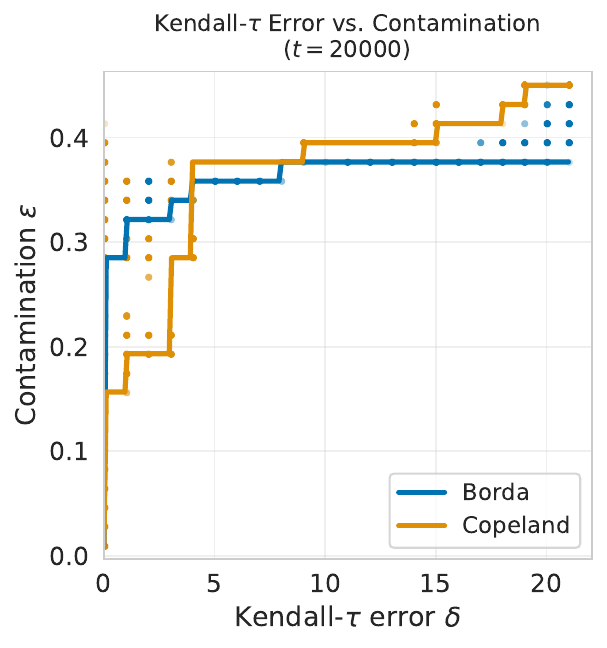}
        \caption{At iteration $t=20000$}
    \end{subfigure}\hfill
    \caption{Average Kendall-$\tau$ error vs. contamination for the Mallows distribution.}
    \label{fig:breakdown_stubborn}
\end{figure*}

\subsection{Local Kemenization as a Kemeny consensus approximation}\label{sec:local_kem_exp}
\paragraph{Experimental setup.} We further evaluate the robustness of classical consensus methods under contamination, as reported in Tables \ref{tab:contaminated_mallows} and \ref{tab:adversarial_mallows}. For simplicity, the analysis is conducted in a centralized setting, though it extends naturally to the decentralized case. For each table, we report for each method the mean Kendall-tau distance $d_{\tau}(\,\cdot\,, \sigma_\star)$ to the true reference ranking of the Mallows distribution (the target we aim to recover), the excess loss $\Delta L_{\widehat P_N}$, and the effect of local Kemenization on $\Delta L_{\widehat P_N}$ as a post-processing step to approximate the Kemeny consensus. We consider two contamination distributions for the corrupted rankings $(\sigma_w)_{w \in B}$. In the random contamination setting, corrupted rankings are sampled i.i.d.\ from the Haar measure $P_{\mathrm{Haar}}$ on $\Sn$. In the adversarial contamination setting, they are sampled i.i.d.\ from a Mallows model $P^{M}_{\theta}$ with $\theta =(\sigma^{\mathrm{rev}},0)$ centered at the reversed ranking $\sigma^{\mathrm{rev}}$.

\paragraph{Results.} Our results as illustrated in Table~\ref{tab:contaminated_mallows} and Table~\ref{tab:adversarial_mallows} show that Kemeny, followed by Copeland, is the most robust consensus method under both contamination models. Notably, applying local Kemenization to either the Borda or Copeland consensus consistently recovers the Kemeny consensus, suggesting that local Kemenization is a promising and lightweight strategy for improving the robustness of simpler consensus methods.
\begin{table}
\centering
\begin{tabular*}{\columnwidth}{@{\extracolsep{\fill}}@{\hspace{0.5em}}lccc}
\toprule
\textbf{Consensus} & $d_{\tau}(\cdot, \sigma_\star)$ & $\Delta L_{\widehat P_N}$ w/o Local Kemenization & $\Delta L_{\widehat P_N}$ w/ Local Kemenization \\
\midrule
Borda    & $0.25 \pm 0.46$ & $0.01 \pm 0.03$ & $0.00 \pm 0.00$ \\
Copeland & $0.07 \pm 0.26$ & $0.00 \pm 0.00$ & $0.00 \pm 0.00$ \\
Footrule & $1.04 \pm 0.81$ & $0.17 \pm 0.17$ & $0.00 \pm 0.00$ \\
Kemeny   & $0.06 \pm 0.24$ & N/A\tablefootnote{For the global Kemeny consensus, the system is already locally Kemeny optimal.} & N/A\footnotemark[\value{footnote}] \\
\bottomrule
\end{tabular*}

\vspace{0.3em}
\caption{Robustness of consensus methods for recovering the central ranking $\sigma_\star$ under random contamination. Rankings are sampled from a Mallows model ($N=100$, $n=8$ items, $\varphi=0.5$) where $\varepsilon=0.3$ of samples are replaced with uniform random permutations. We report the Kendall-tau distance to the ground truth $d_{\tau}(\cdot, \sigma_\star)$ and the excess risk $\Delta L_{\widehat P_N}$ follows the definition in Equation~\eqref{eq:excess_risk_kem}, where $L_{\widehat P_N}^\star$ is the Kemeny optimal value for the empirical distribution $\widehat P_N$. Results show mean $\pm$ standard deviation over $100$ trials.}
\label{tab:contaminated_mallows}
\end{table}

\begin{table}
\centering
\begin{tabular*}{\columnwidth}{@{\extracolsep{\fill}}@{\hspace{0.5em}}lccc}
\toprule
\textbf{Consensus} & $d_{\tau}(\cdot, \sigma_\star)$ & $\Delta L_{\widehat P_N}$ w/o Local Kemenization & $\Delta L_{\widehat P_N}$ w/ Local Kemenization \\
\midrule
Borda    & $0.91 \pm 0.85$ & $0.03 \pm 0.04$ & $0.00 \pm 0.00$ \\
Copeland & $0.59 \pm 0.71$ & $0.00 \pm 0.02$ & $0.00 \pm 0.01$ \\
Footrule & $1.26 \pm 0.99$ & $0.13 \pm 0.11$ & $0.01 \pm 0.02$ \\
Kemeny   & $0.58 \pm 0.67$ & N/A\footnotemark[\value{footnote}] & N/A\footnotemark[\value{footnote}] \\
\bottomrule
\end{tabular*}
\vspace{0.3em}
\caption{Robustness of consensus methods under adversarial contamination. Setup is identical to Table~\ref{tab:contaminated_mallows}, except contaminated samples are drawn from a Mallows model centered at the reversed ranking $\sigma_\mathrm{rev}$ rather than uniform random permutations.}
\label{tab:adversarial_mallows}
\end{table}

\section{Experiments compute resources}
\label{sec:compute}
All experiments were run on a single compute node equipped with one NVIDIA P100 16GB GPU, 30GB RAM, and 20 CPU cores. Despite GPU availability, all computations are CPU-bound (Numpy/Numba), with parallelism limited to 8 workers across trial configurations. In Table~\ref{tab:compute_time}, we describe approximately the time to execute the different numerical simulations.

\begin{table}
\centering
\begin{tabular}{lc}
\toprule
\textbf{Experiments} & Time to compute \\
\midrule
Debian / Inference model convergence & $< 5$min\\
Sushi convergence & $\sim 1$h\\
Mixture of Mallows / MovieLens convergence & $\sim 15$min\\
Experiment~\ref{sec:breakdown_expe} (Breakdown studies) & $\sim3$h\\
Experiment~\ref{sec:synth-data-conv} (Synthetic data with Spearman Footrule) & $\sim 1$h\\
Experiment~\ref{sec:comparison_rates} (Comparison of convergence rates)& $\sim 1$h\\
Experiment~\ref{sec:ablation_mallows} (Varying dispersion parameter) & $\sim 1$h$05$\\
Experiment~\ref{sec:sweep_n_N}
 (Sweep for different $n$ and $N$)
 &$\sim 3$h\\
Experiment~\ref{sec:local_kem_exp} (Local Kemenization study) & $\sim 30$min\\
\bottomrule
\end{tabular}
\vspace{0.3em}
\caption{Time to compute per simulation.}
\label{tab:compute_time}
\end{table}


\end{document}